\documentclass[lettersize,journal]{IEEEtran}

\usepackage{graphicx}
\usepackage{amsmath, amssymb}
\usepackage[numbers, sort & compress]{natbib}
\usepackage[colorlinks=true,
            linkcolor=blue,     %
            citecolor=blue,     %
            urlcolor=blue       %
           ]{hyperref}
\hypersetup{pdftitle={XIT: Exploration and Exploitation Informed Trees for Active Gas Distribution Mapping in Unknown Environments},
            pdfauthor={Mal Fazliu, Matthew Coombes, Sen Wang, and Cunjia Liu}}
\expandafter\def\expandafter\UrlBreaks\expandafter{\UrlBreaks\do\-}
\usepackage[acronym,nopostdot]{glossaries}
\usepackage{changepage} %
\usepackage[linesnumbered,ruled,vlined]{algorithm2e}

\usepackage{xcolor}
\usepackage{array}
\glsdisablehyper
\newacronym{gdm}{GDM}{gas distribution mapping}
\newacronym{ipp}{IPP}{informative path planning}
\newacronym{nbt}{NBT}{next-best trajectory}
\newacronym{ucb}{UCB}{upper confidence bound}
\newacronym{slam}{SLAM}{simultaneous localisation and mapping}
\newacronym{ogm}{OGM}{occupancy grid mapping}
\newacronym{pomdp}{POMDP}{partially observable Markov decision process}
\newacronym{gmrf}{GMRF}{Gaussian Markov random field}
\newacronym{gbp}{GBP}{Gaussian belief propagation}
\newacronym{xit}{XIT}{Exploration and Exploitation Informed Trees}
\newacronym{wgfd}{WGFD}{Wavefront Gas Frontier Detection}
\newacronym{bit*}{BIT*}{batch informed trees}
\newacronym{wfd}{WFD}{wavefront frontier detection}
\newacronym{ffd}{FFD}{fast frontier detection}
\newacronym{ros}{ROS}{Robot Operating System}
\newacronym{bfs}{BFS}{breadth-first search}
\newacronym{gsl}{GSL}{gas source localisation}
\newacronym{kld}{KLD}{Kullback-Leibler divergence}

\newtheorem{theorem}{Theorem}        
\newtheorem{lemma}[theorem]{Lemma}            
\newtheorem{proposition}[theorem]{Proposition}

\newtheorem{remark}[theorem]{Remark}

\title{XIT: Exploration and Exploitation Informed Trees for Active Gas Distribution Mapping in Unknown Environments}

\author{Mal Fazliu,
        Matthew Coombes,
        Sen Wang,
        and Cunjia Liu%
\thanks{Mal Fazliu's contribution to this work was supported by an Engineering and Physical Sciences Research Council (EPSRC) studentship under project reference 2886366.
\textit{(Corresponding author: Cunjia Liu.)}}%
\thanks{M. Fazliu, M. Coombes, and C. Liu are with the Department of Aeronautical and Automotive Engineering,
Loughborough University, Loughborough LE11 3TU, U.K.
(e-mails: \{m.fazliu2, m.coombes, c.liu5\}@lboro.ac.uk).}%
\thanks{S. Wang is with the Department of Electrical and Electronic Engineering, Imperial College London, London SW7 2AZ, U.K.
(e-mail: sen.wang@imperial.ac.uk).}%
}

\AddToHook{shipout/background}{%
  \ifnum\value{page}=1
    \put(0,-\paperheight){\raisebox{0.07in}{\makebox[\paperwidth][c]{\parbox[b]{\dimexpr\paperwidth-0.8in\relax}{\scriptsize\centering
    \copyright~2026 IEEE. Personal use of this material is permitted. Permission from IEEE must be obtained for all other uses, in any current or future media, including reprinting/republishing this material for advertising or promotional purposes, creating new collective works, for resale or redistribution to servers or lists, or reuse of any copyrighted component of this work in other works. Accepted for publication in IEEE Transactions on Robotics.}}}}%
  \fi}
\begin{document}
\maketitle

\begin{abstract}

Mobile robotic gas distribution mapping (GDM) provides critical situational awareness during emergency responses to hazardous gas releases. However, most systems still rely on teleoperation, limiting scalability and response speed. Autonomous active GDM is challenging in unknown and cluttered environments, because the robot must simultaneously explore traversable space, map the environment, and infer the gas distribution belief from sparse chemical measurements. We address this by formulating active GDM as a next-best-trajectory informative path planning (IPP) problem and develop XIT (Exploration and Exploitation Informed Trees), a sampling-based planner that grows a single shared tree toward multiple frontier goals, coupling exploration of unresolved regions with informative gas sampling en route. XIT draws a batch of samples from an Upper Confidence Bound (UCB) information field derived from the current gas posterior and expands the tree using a cost that trades off travel effort against information acquisition. To enable plume-aware exploration, we introduce the gas frontier concept, defined as unobserved regions adjacent to high gas concentrations, and propose the Wavefront Gas Frontier Detection (WGFD) algorithm for their identification. High-fidelity simulations demonstrate the benefits of XIT in terms of GDM quality and efficiency, and a scaled real-world experiment demonstrates its practicality. Although developed for active GDM, XIT is applicable to other robotic information-gathering tasks in unknown environments that face the exploration and exploitation trade-off.

\end{abstract}

\begin{IEEEkeywords}
    Gas distribution mapping, active sensing, informative path planning, next-best-trajectory, frontier exploration. 
\end{IEEEkeywords}

\section{Introduction}

In the event of a gas leak, timely and reliable acquisition of gas dispersion information is critical to assessing risk and responding effectively; any delay can have serious consequences for human safety and environmental integrity \cite{Murphy2012}. At the scene, gas concentration levels are regarded as the most valuable raw data that can be measured. To gather this information, mobile robotic platforms equipped with chemical and perceptual sensors are a standard choice, as they can access polluted and uncertain areas unsafe for humans and withstand harsher conditions. Still, these platforms offer limited value without a tool to localise each concentration reading and reconstruct a gas distribution model. \Gls{gsl} methods, such as source term estimation, face challenges in unknown environments where structural mapping is required, and are further limited in indoor or cluttered settings without a prevailing wind field or when dealing with complex multi-source releases \cite{Francis2022}. In such scenarios, \gls{gdm} provides an alternative by constructing spatial maps of gas concentration from spatiotemporal chemical measurements.

Many \gls{gdm} algorithms have been developed for robotic gas sensing and mapping \cite{Lilienthal2004GDM,Stachniss137562,Lilienthal2009,Monroy2016,wiedemann2018multi,Ojeda2023}. One popular representation is the \gls{gmrf} model \cite{Monroy2016}, as it is able to efficiently encode gas measurements and spatial dispersion within a factor graph framework, in which obstacle information shapes the graph connectivity, making it suitable for cluttered or indoor environments. Recently, \cite{Rhodes2023STRUC} solved this factor graph online using \gls{gbp} \cite{Rhodes2022}, demonstrating scalable real-time \gls{gdm} in both two and three dimensions in unknown environments through performing structural mapping on a teleoperated mobile robot.

While remote operation demonstrates the potential of online robotic \gls{gdm} in realistic scenarios, it relies on communication and requires specialists on hand, which can slow response times in critical situations. There is therefore strong incentive to pursue an autonomous \gls{gdm} solution that offers faster response and improved performance through \gls{ipp} \cite{bai2021information,POPOVIC2024104727}, that is, an active sensing principle that closes the loop by actively selecting sensing locations as the gas distribution belief and structural mapping evolve. We refer to this closed-loop paradigm as \emph{active} \gls{gdm}.

\begin{figure}[t]
    \centering
    \includegraphics[width=0.85\linewidth]{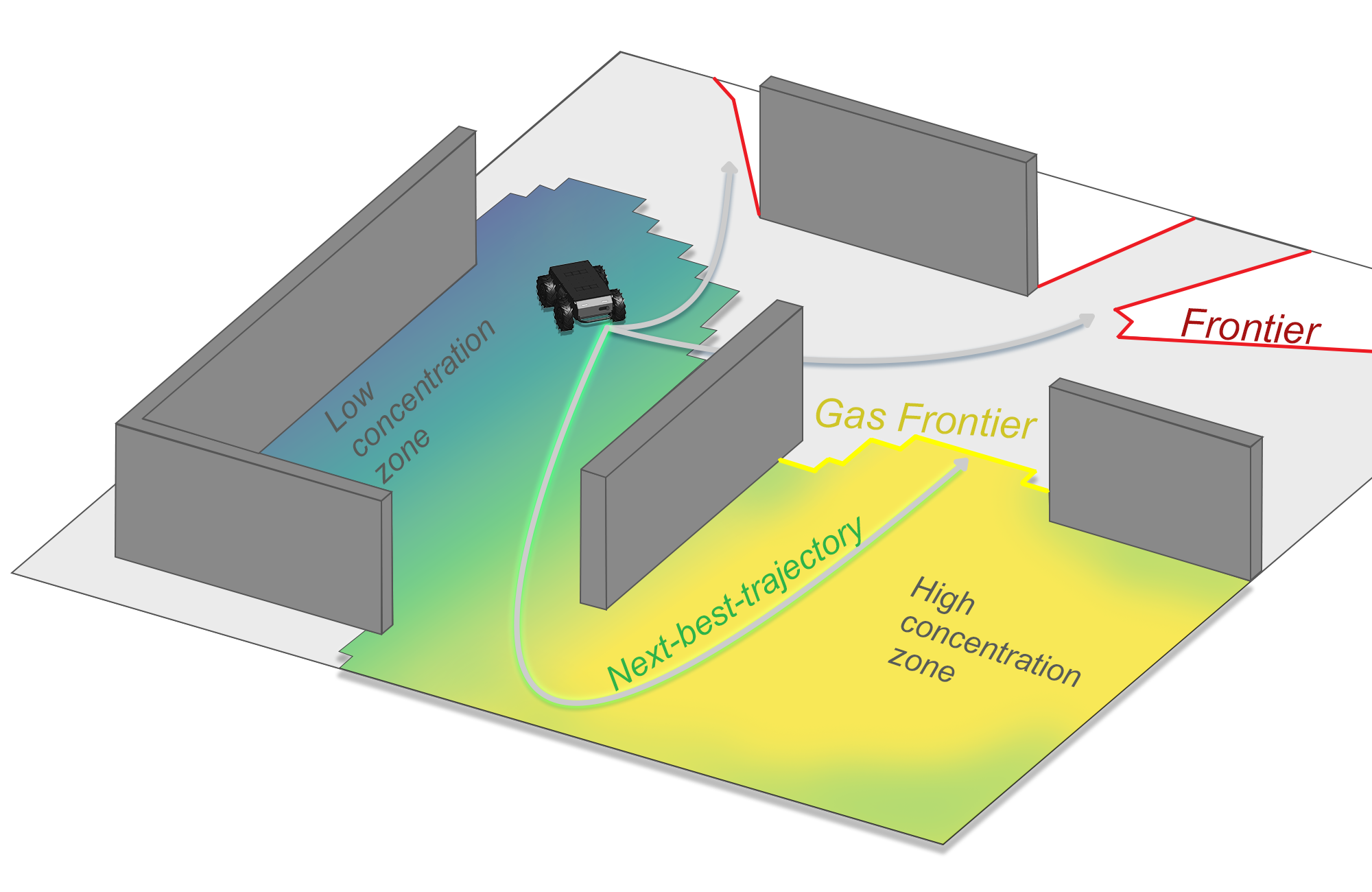}
    \caption{Overview of the XIT planner for active GDM. The figure illustrates a robot constructing a gas distribution map from concentration measurements in real time and planning gas-exploitative, distance-aware trajectories toward both gas frontiers and environmental frontiers. The three candidate paths generated by XIT are shown in grey, with the next-best trajectory (NBT) highlighted in green, the path with the highest expected information value, balancing gas dispersion inference and spatial exploration where possible.}
    \label{fig:intro_overview}
\end{figure}

Previous studies on \gls{ipp} for \gls{gdm} focus on gas inference and address exploiting gas-rich regions and gas map coverage, but often fall short in real-world settings due to two main limitations: (i) the computational or modelling limitations of the \gls{gdm} function and (ii) the assumption of a known map of the environment. As discussed, the former has lately been mitigated through scalable inference using \gls{gbp} for the \gls{gmrf} \cite{Rhodes2023STRUC}. The latter remains a central challenge for active \gls{gdm}, where the robot must make informative sensing decisions while its traversable workspace is still being discovered. 

 More broadly, recent active mapping and \gls{ipp} methods have made substantial progress in unknown environments, including tree-based next-best-trajectory or next-best-view planners \cite{Schmid8968434,Lindqvist2024,Border2024IJRR}, semantic information gathering and mapping \cite{Asgharivaskasi2023,Ruckin2023}, and adaptive sampling-based planning for environment monitoring \cite{Hollinger2014,GhaffariJadidi2019,Moon2026TRO}. Learning-based approaches have also been developed for adaptive target discovery and multi-robot exploration \cite{Vashisth2024,Cheng2025TRO}. These methods address important geometric, semantic, and task-specific information-gathering problems. However, active \gls{gdm} requires the joint treatment of unknown-environment exploration, obstacle-aware gas inference, and trajectory-level gas information acquisition. A realistic planner must therefore uncover traversable space while deciding where to collect informative gas measurements.

To address this gap, we propose an active \gls{gdm} framework for unknown and cluttered environments, where the robot must reason about multimodal sensor information and balance exploration and exploitation objectives. These objectives include:
(i) environmental exploration, encouraging the robot to explore its surroundings to build a geometric map and uncover areas that may contain gas;
(ii) gas structure exploration, probing newly discovered gas regions to explore the extent and shape of the plume; and
(iii) gas exploitation, revisiting high-value gas regions to improve confidence in the \gls{gdm} model.
An illustration depicting the mobile robot performing active \gls{gdm} with the solution presented in this paper is included in Fig.~\ref{fig:intro_overview}. Our strategy is structured around guiding the robot toward exploration-rich destinations while simultaneously exploiting valuable information en route. Inspired by the decoupling principle for active \gls{slam} in \cite{Placed2023}, the proposed framework first generates a set of informative paths to each exploratory goal, then selects the \gls{nbt} to execute.

At the core of this framework, we introduce \gls{xit}, a multi-objective, sampling-based planner adapted from \gls{bit*} \cite{Gammell2015} that facilitates our \gls{nbt} approach for active \gls{gdm} in unknown environments. \gls{xit} constructs quasi-simultaneous paths toward goal regions selected for their exploration potential, while exploiting key gas regions en route through information-guided sampling and tree expansion, such that each generated trajectory jointly addresses exploration and exploitation objectives. Sampling is performed over an \gls{ucb} map derived from the current gas state estimate, and tree expansion is guided by a \gls{ucb}–distance cost function that favours regions of high gas concentration, high uncertainty, and low travel cost. By integrating informed sampling with guided tree expansion, \gls{xit} limits expansion into less informative areas, thereby reducing unnecessary search effort and yielding high-quality paths to each goal.

The main contributions of this work are summarised as follows:
\begin{itemize}
    \item \textbf{\gls{xit}}, a sampling-based \gls{ipp} planner that constructs multi-goal informed trees over a \gls{ucb}-based information field to trade off information acquisition and travel cost, with graph optimality on fixed sample-induced graphs and probabilistic completeness and asymptotic optimality established under the stated assumptions;
    \item To our knowledge, the first \textbf{active \gls{gdm} system} that jointly reasons about environmental exploration, gas structure exploration, and gas exploitation, evaluated in high-fidelity simulation and demonstrated in a scaled real-world gas-release experiment;
    \item \textbf{gas frontier}, a novel concept for detecting plume boundaries to support plume-aware exploration, and the {\gls{wgfd}} algorithm for their identification.
\end{itemize}

\section{Related Work}
Path planning approaches for \gls{gdm} have historically utilised predetermined coverage trajectories, such as lawnmower sweeps, that do not adapt based on observed gas distributions \cite{Lilienthal2009, Neumann2012, Luo2015, Burques2019Nano}. These fixed patterns fail to efficiently leverage the impact of measurement locations on \gls{gdm} accuracy. \Gls{ipp} addresses this by adaptively selecting sensing positions that evolve with the current gas distribution belief.

Early \gls{ipp} approaches for \gls{gdm} iteratively selected individual sampling locations based on the current gas distribution estimate. The first explicit use of \gls{ipp} for \gls{gdm} appears in \cite{Neumann2012microdrones}, where an adaptive sensing strategy based on artificial potential fields balanced exploitation of high-concentration and high-variance regions with gas map coverage via repulsive forces from previous measurements, converging faster than fixed sweeps. They utilised the Kernel DM \cite{Lilienthal2004GDM} for \gls{gdm}, a lightweight data-driven method that was popular at the time, in particular the DM+V/W variant \cite{Reggente2009} which incorporates wind information but not temporal dependencies. The main drawback of Kernel DM (and its variants) has been its inconvenience in accounting for obstacle effects on the plume. In \cite{Rhodes2020}, several reward functions were examined using the original \gls{gmrf} formulation \cite{Monroy2016}, which permits effective obstacle modelling, facilitating \gls{ipp} for \gls{gdm} in cluttered environments to be tested and validated using a gas release CFD dataset. A* search was employed to compute distance-based traversal costs between the sampling goal locations. The most effective reward combined concentration with time-weighted joint uncertainty, yielding substantially faster convergence than a non-adaptive sweep. \citet{Ercolani2022} focused on 3D \gls{gdm} and opted for the Kernel DM+V/W algorithm, citing its lightweight computational demands and the lack of a scalable solver for \gls{gmrf} at the time. They evaluated \gls{ipp} methods that prioritised uncertainty and variability rather than exploiting high-concentration regions, including Lévy flight, entropy-based, and \gls{kld} criteria. While these strategies produced detailed maps over limited areas, they exhibited suboptimal exploratory behaviour, becoming trapped in locally informative regions and failing to achieve broader coverage. To address this limitation, they introduced a clustering algorithm to partition the operational space into clusters to visit sequentially, enforcing broader coverage while maintaining adaptive sampling within each cluster. 

More recent work has addressed planning with non-myopic forward simulation \cite{Gongora2023} and multi-waypoint trajectory execution \cite{Nanavati2024}. A \gls{gsl} technique known as infotaxis was adapted to \gls{gdm} in \cite{Gongora2023} using \gls{pomdp} planning with the gas and wind (GW) variant of the \gls{gmrf} \cite{Gongora2020}, which was computed using a sparse direct solver. The approach performs finite-horizon forward simulation to evaluate multi-step action sequences in the four cardinal directions (N, W, S, E) using a \gls{kld}-based reward, executing only the first action of the planned sequence before replanning in a receding-horizon manner. Forward simulation requires evaluating GW-\gls{gmrf} multiple times per planning cycle, which the authors acknowledge restricts real-time applicability. More broadly, GDM formulated as a POMDP can be addressed using general-purpose Monte Carlo tree search (MCTS) solvers \cite{mern2021bayesian,mern2021improved}; however, these operate at the level of sequential action selection rather than collision-free informative trajectory generation in a partially known environment. A two-phase low-to-high resolution framework using \gls{gmrf} with \gls{gbp} was proposed in \cite{Nanavati2024}, first performing sparse sampling to identify high-concentration regions, then selecting the highest-reward cells and planning distance-optimised trajectories through them for dense sampling within fixed time budgets. While this framework executes complete multi-waypoint paths, it rigidly separates exploration and exploitation into sequential phases rather than considering both objectives within each trajectory. Recent learning-based \gls{gdm} has also been considered in GDM-Net++ \cite{Kulbaka2025IROS}, which combines deep Q-learning with Gaussian-process regression in a multi-robot setting. This represents a distinct formulation from the obstacle-aware online active-GDM setting considered here, based on \gls{gmrf} inference with \gls{gbp}.

Of the aforementioned works, \cite{Neumann2012microdrones}, \cite{Ercolani2022}, and \cite{Nanavati2024} do not consider obstacles in their planning, while \cite{Rhodes2020} and \cite{Gongora2023} assume prior knowledge of the environment via predetermined occupancy maps. These assumptions limit applicability to realistic emergency response scenarios, where unknown and cluttered environments necessitate joint reasoning about environmental exploration in conjunction with effective gas distribution inference. Even \gls{gsl}, the alternative approach to gas distribution modelling that has been more extensively studied \cite{Francis2022}, struggles to address the full problem. Recent work, such as \cite{Wang10665938, Kim2025DualModeGSL}, has extended \gls{ipp} to \gls{gsl} in unknown environments, yet such methods are typically formulated assuming a fixed number of sources, most often a single source. This limits generalisability when source configurations are unknown, even in state-of-the-art approaches that leverage \gls{gdm} to update beliefs over source parameters in known environments \cite{Ojeda2023}. The only work to address unknown environments for \gls{gdm} is \cite{Prabowo2022}, which performs online \gls{ogm} alongside Kernel DM+V \cite{Lilienthal2009} for \gls{gdm}. Their system uses a finite state machine that switches between frontier exploration and gas exploitation based on manually tuned thresholds. In gas mode, a single high-cost cell is selected and an RRT-based planner generates a collision-free path to it. While the planner accounts for obstacles, the underlying Kernel DM+V model does not, and unlike \cite{Nanavati2024}, only the goal location is informed by the current gas belief, not the trajectory.

 In this work, we propose \gls{xit}, a sampling-based planner that generates candidate informed trajectories considering both environmental exploration and gas inference objectives, and selects the most informative trajectory for execution. This aligns with the next-best-trajectory concept \cite{Lindqvist2024}, introduced for autonomous exploration in unknown environments, where RRT-based expansion generates multiple frontier-reaching paths and trajectories are selected by trading off information gain against distance and actuation cost. \gls{xit} generalises this idea into a more unified and ordered approach for active \gls{gdm}, growing one shared tree toward both occupancy and gas frontiers. Moreover, unlike \cite{Lindqvist2024}, tree growth in \gls{xit} is informed, with gas exploitation incorporated directly into sampling and tree expansion through a \gls{ucb}-based information field. Unlike standard \gls{bit*} \cite{Gammell2015}, which is formulated around a single-goal planning query, \gls{xit} evaluates multiple frontier-derived target regions within one shared search structure. Crucially, by operating within an obstacle-aware \gls{gdm} formulation, \gls{xit} does not require assumptions about the number or locations of gas sources and is applicable to complex multi-source releases.

\section{Active GDM: Formulation and System}
\label{sec:active-gdm-formulation-system}
This section formalises the representations and planning concepts used throughout our active GDM framework. We first describe how the environment and gas fields are modelled, then introduce the frontier-based goal regions that drive exploration, and finally outline the next-best-trajectory formulation and the overall system architecture (see Fig. \ref{fig:overview}) that support online decision making.

\begin{figure*}[t]
\begin{adjustwidth}{0cm}{-0.6cm}
    \centering
    \begin{minipage}[t]{0.5\textwidth}
        \centering
        \includegraphics[width=\linewidth, height=6.9cm]{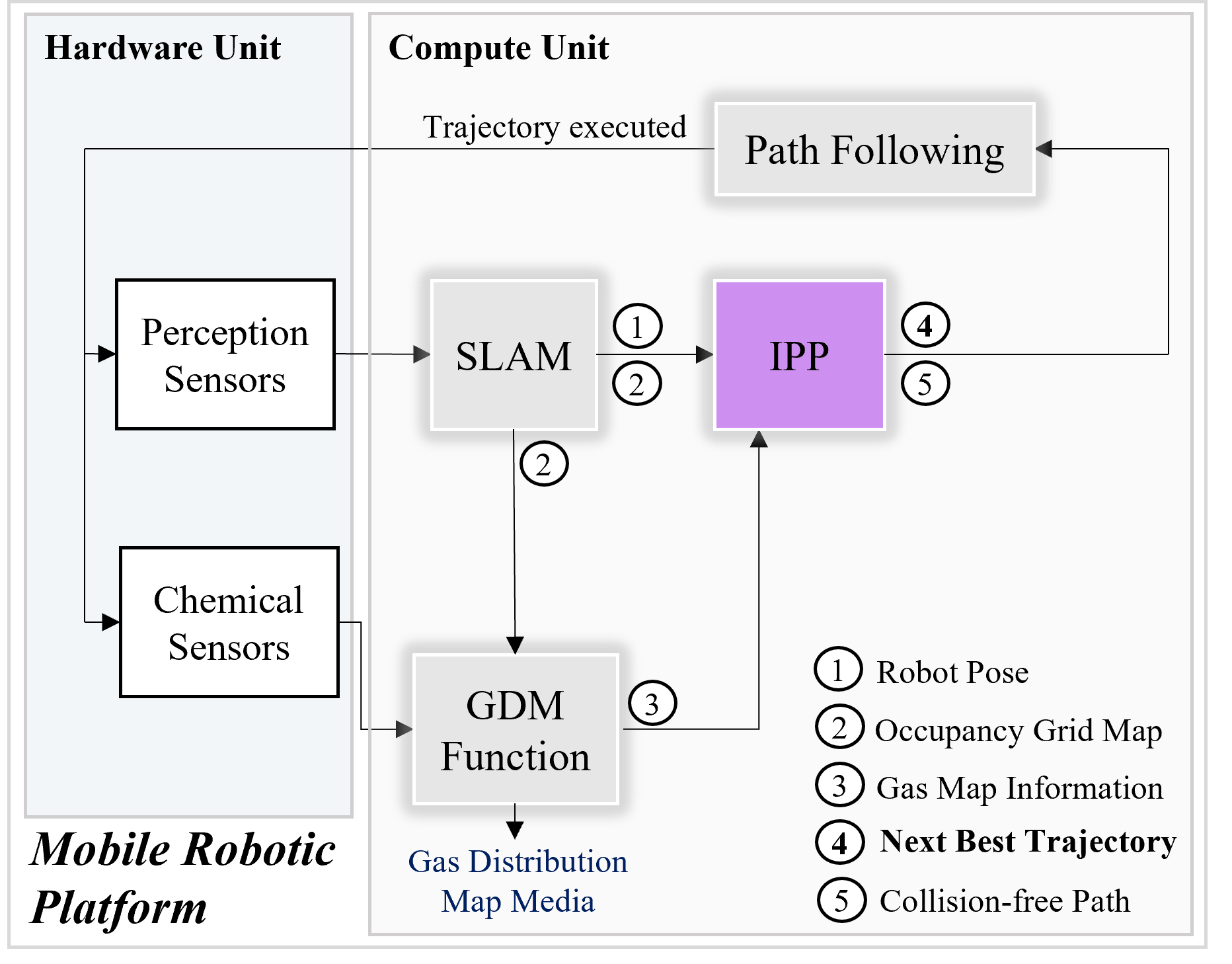}
        \scriptsize \textbf{(a)}
        \label{fig:overview2}
    \end{minipage}%
    \hspace{-2.5mm}
    \begin{minipage}[t]{0.5\textwidth}
        \centering
         \includegraphics[width=\linewidth, height=7.1cm]{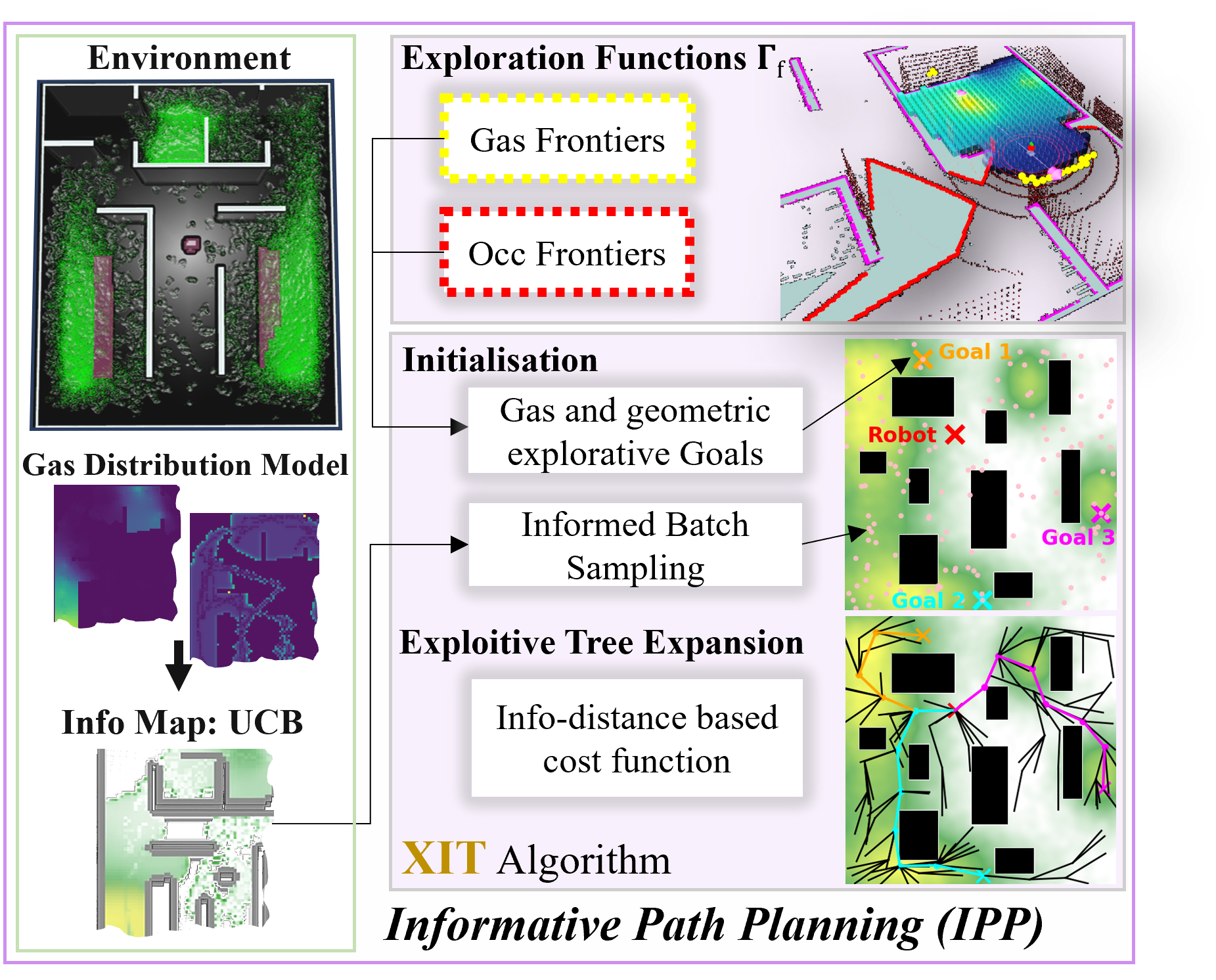}
        \scriptsize \textbf{(b)}
        \label{fig:overview3}
    \end{minipage}
    \caption{System overview of the proposed active \gls{gdm} framework. On the left (a) is a general system architecture for the mobile robotic application. Key functional components in the compute unit include localisation and mapping for navigation and occupancy information (summarised as the \gls{slam} module), the \gls{gdm} function, and the \gls{ipp} module, which selects the \gls{nbt}. The \gls{ipp} process is expanded in (b), where we also illustrate a gas-polluted environment and the corresponding mean and uncertainty maps produced during \gls{gdm}. The concept of gas frontiers is visualised using yellow markers along the boundary of the estimated gas plume, while conventional frontiers outlining the unknown occupancy map are in red. Both types are used by \gls{xit} to steer and terminate tree expansion in exploration-rich regions, while a \gls{ucb}-based cost function, derived from the current gas model, guides paths through inferred high-value gas areas. A separate example is included to illustrate the characteristic trajectories and tree structure generated by \gls{xit}.}
    \label{fig:overview}
\end{adjustwidth}
\end{figure*}
\subsection{Environment and Gas Representation}
\label{sec:envioronment representation}

We consider the problem in a planar configuration space with unknown, static obstacles. Let the robot’s configuration space be $\mathcal{X} \subset \mathbb{R}^2$, with $\mathcal{X}_{\mathrm{obs}} \subset \mathcal{X}$ denoting obstacle regions and $\mathcal{X}_{\mathrm{free}} := \mathcal{X} \setminus \mathcal{X}_{\mathrm{obs}}$ denoting the true free space. Let $\mathbf{x}_k \in \mathcal{X}_{\mathrm{free}}$ represent the robot’s state at planning time step $k \in \mathbb{N}$. To represent the environment structure, we maintain an occupancy grid map $\mathcal{M}_k^{\mathrm{occ}}$ on a lattice $\mathcal{C}$, where each cell $c \in \mathcal{C}$ has an occupancy belief $p_k(c)$. Each state $\mathbf{x} \in \mathcal{X}$ corresponds to a unique cell $c(\mathbf{x}) \in \mathcal{C}$, which allows us to write map quantities directly as functions of $\mathbf{x}$, e.g., $p_k(\mathbf{x}) := p_k(c(\mathbf{x}))$. The known-free subset at step $k$ is the union of cells whose occupancy probability is below a free threshold $\tau_{\mathrm{free}}$, i.e.,

\begin{equation}
\hat{\mathcal{X}}_{\mathrm{free}}(k) := \bigcup_{c:\, p_k(c) \le \tau_{\mathrm{free}}} c
\subset \mathcal{X}_{\mathrm{free}} .
\label{eq:known-free}
\end{equation}

In this work, planning is restricted to $\hat{\mathcal{X}}_{\mathrm{free}}(k)$; unknown cells are not traversed. The robot pose is assumed to be available, for example from a \gls{slam} solution. Given such a pose estimate, an occupancy grid map can be constructed using LiDAR sensors through a standard inverse sensor model with ray tracing. The uncertainty of the \gls{slam} output is not explicitly propagated through \gls{xit}, as it is assumed to be small relative to the uncertainty associated with gas sensing and plume variability in the considered settings.

Analogous to the occupancy map, the gas distribution $\mathcal{M}_{k}^{\mathrm{gas}}$ is represented by a \gls{gmrf} posterior $(\boldsymbol{\mu}_k, \boldsymbol{\Sigma}_k)$ defined on the same lattice (or a coarser one). This posterior can be obtained by processing gas sensor data collected up to step $k$ using, e.g., direct solvers \cite{Monroy2016} or belief propagation \cite{Rhodes2023STRUC}. The \gls{gmrf} representation further supports per-state statistics, with mean $\mu_k(\mathbf{x})$ and variance $\varepsilon^2_k(\mathbf{x})$ available on $\hat{\mathcal{X}}_{\mathrm{free}}(k)$; cells marked occupied or unknown in $\mathcal{M}_k^{\mathrm{occ}}$ add no variables to the graph, but sever the smoothness couplings that would otherwise cross them.

\subsection{Frontier-based Goal Regions}
\label{sec:frontier-based goal regions}
To enable exploration of both the unknown physical environment and the gas structure, we adopt frontier detection techniques \cite{Keidar2014EfficientFrontier} to identify goal regions for the robot to move toward. The occupancy frontiers $\mathcal{F}_{k}^{\mathrm{occ}}$ denote the set of frontier clusters, where each frontier $f^{\mathrm{occ}} \in \mathcal{F}_{k}^{\mathrm{occ}}$ is a contiguous group of cells on the boundary between known and unknown space in the occupancy map. These regions drive the expansion of the robot’s reachable free space. Inspired by the role of frontiers in guiding environment exploration, we introduce \emph{gas frontiers}, denoted $\mathcal{F}_{k}^{\mathrm{gas}}$, defined analogously as a set of gas-frontier clusters $f^{\mathrm{gas}} \in \mathcal{F}_{k}^{\mathrm{gas}}$, each representing the boundary of a high-concentration gas region adjacent to unobserved gas states. A detailed description of the gas-frontier detection procedure is provided in Section~\ref{sec:gas_frontier}.

The exploration-rich goal set is then
\begin{equation}
\mathcal{G}_k = \mathcal{F}_{k}^{\mathrm{occ}} \cup \mathcal{F}_{k}^{\mathrm{gas}} .
\label{eq:goal-set}
\end{equation}
That is, the collection of all occupancy and gas frontier clusters.

\subsection{Next-best Trajectory Planning}
\label{sec:Next-best trajectory planning}
To formalise the planning problem, we model an admissible trajectory as an absolutely continuous curve 
\[
\sigma:[0,1] \to \hat{\mathcal{X}}_{\mathrm{free}}(k),
\]
starting at the robot's current state, $\sigma(0) = \mathbf{x}_k$, and terminating in one of the frontier-based goal regions, $\sigma(1) \in \mathcal{G}_k$. 
Our goal is to choose a trajectory $\sigma$ from $\mathbf{x}_k$ that enables the robot to efficiently collect observations along the path, improving both the occupancy grid map and the gas distribution model; in particular, expanding known-free space (environment exploration), expanding the critical gas region (gas structure exploration), and reducing gas uncertainty around high-concentration regions (gas exploitation).

Executing a trajectory produces new measurements, and these updates are captured by an online mapping operator
\begin{equation}
\Phi : ({\mathcal{M}_k}^{(\cdot)};\, \sigma) \mapsto {\mathcal{M}_{k+1}}^{(\cdot)},
\label{eq:mapping-operator}
\end{equation}
which performs the occupancy-grid update and the \gls{gmrf} inference based on the data gathered along $\sigma$.

Full non-myopic \gls{ipp} for gas mapping can be cast as a POMDP \cite{Gongora2023}, but requires predicting the coupled evolution of occupancy and gas beliefs and future measurements. This will inevitably introduce substantial computation overhead and model uncertainties in an unknown environment. As a first solution to the active \gls{gdm} problem, our premise is that the robot obtains more reliable information by acting on its current beliefs and re-planning in real time after executing the selected trajectory, rather than relying on non-myopic predictions over uncertain futures. We therefore adopt a myopic, information-aware \gls{nbt} formulation. At each planning iteration $k$, the robot selects the next best trajectory 
\begin{equation}
\sigma_k^*
=
\arg\min_{\sigma \in \Sigma_k}
J_k(\sigma),
\quad
\text{s.t. } 
\sigma(0) = \mathbf{x}_k,\;
\sigma(1) \in \mathcal{G}_k .
\label{eq:nbt-opt}
\end{equation}
Here, the set $\Sigma_k$ contains only collision-free trajectories, and $J_k(\sigma)$, as specified in Section~\ref{subsec:xit_exploitation}, is a dedicated cost function that trades off travel distance against gas-related information and is additive and monotonic over feasible trajectory extensions. Exploration of new space and critical gas structure is induced through the frontier-based goal set $\mathcal{G}_k$. After executing $\sigma_k^*$ in full, i.e., traversing the selected trajectory to its goal rather than a single step of it, the models are updated via the mapping operator $\Phi$ in Eq.~\eqref{eq:mapping-operator}, yielding $\mathcal{M}_{k+1}$, and the optimisation in Eq.~\eqref{eq:nbt-opt} is then resolved at iteration $k{+}1$.

This recursive \gls{nbt} formulation captures the autonomy principle of our system: trajectories are guided through information-rich regions while maintaining progress toward frontier-based goals. The key components of the optimisation in Eq.~\eqref{eq:nbt-opt}, namely the candidate set $\Sigma_k$, the cost function $J_k(\sigma)$, and the frontier-derived goal set $\mathcal{G}_k$, must be designed cooperatively to ensure effective environmental exploration and gas inference within active \gls{gdm}.

\subsection{System Architecture and Components}

Before presenting the algorithms in detail, we provide a high-level overview of the proposed active \gls{gdm} system. As shown in Fig.~\ref{fig:overview}, the framework combines gas sensing, \gls{ogm}, online \gls{gdm}, \gls{ipp} via \gls{xit}, and a path-following module to execute the selected trajectory.

At the core of this architecture is the \gls{ipp} formulation in Eq.~\eqref{eq:nbt-opt}, which determines the next best trajectory based on frontier-derived goal regions and a cost function that promotes gas exploitation while remaining sensitive to travel cost. The associated sample-based solver, \gls{xit}, constructs a single information-guided tree whose branches grow toward all selected goal regions in parallel fashion, expanding only through known-free space and avoiding obstacles.

Unlike \gls{gsl} approaches that face challenges in unknown environments, particularly indoors or under multi-source releases \cite{Francis2022}, \gls{xit} is well suited to active gas inference in unknown environments without requiring assumptions about source number or location. Its tree exploits informative gas regions en route, while frontier-based goal selection promotes gas-structure or environmental exploration at the destination.

\section{\texorpdfstring{Exploration and Exploitation Informed Trees}{Exploration and Exploitation Informed Trees}}

The design of \gls{xit} follows the view of \gls{bit*}~\cite{Gammell2015} as an ordered search over a sample-induced implicit random geometric graph. Different from the single-goal geometric setting typically considered in sampling-based optimal planning, \gls{xit} is adapted to active GDM through information-biased sampling, multiple frontier-derived goal regions, and a cost-to-come update and rewiring mechanism. This allows the graph-optimal form of the algorithm to search the current sample-induced graph until no further cost improvement is available. In the online implementation, the same search structure can be terminated once feasible goal-reaching trajectories are found, or when a planning-time budget is reached.

The following subsections define the notation, describe the \gls{xit} search procedure, specify the design choices on information field and cost function, and discuss the resulting completeness and optimality properties.

\subsection{XIT: Notation}
\label{sec:notation}
The notation used in the formulation of \gls{xit} (Alg.~\ref{AlgX}) is defined in this subsection, whereas the notation introduced in Section~\ref{sec:active-gdm-formulation-system} is carried over where applicable. The \gls{xit} planner is executed at fixed instant $k$ as shown in the active \gls{gdm} formulation in Eq.~\eqref{eq:nbt-opt}. For notational simplicity, we drop the index \(k\) on all quantities except \(\mathbf{x}_k\), which denotes the current robot state. Thus, the known-free space $\hat{\mathcal{X}}_{\mathrm{free}}$, gas map $\mathcal{M}^{\mathrm{gas}}$, and frontier-derived goal regions $\mathcal{G}$ are all treated as fixed during a single planning call.

\gls{xit} operates over a finite sample set \(\mathbf{X}_{\mathrm{s}}\subset \hat{\mathcal{X}}_{\mathrm{free}}\), generated by the informed sampling procedure introduced in Section~\ref{subsec:xit_exploitation}. Together with the root state \(\mathbf{x}_k\), these samples induce an implicit random geometric graph (RGG), where two samples may be connected if their Euclidean distance is no greater than the connection radius \(r_N\). XIT incrementally constructs an explicit search tree $\mathcal{T} := (V,E)$, where \(V\subseteq \mathbf{X}_{\mathrm{s}}\cup\{\mathbf{x}_k\} \) is the set of vertices currently in the tree and \(E=\{(\mathbf{v},\mathbf{w})\}\) is the set of directed parent--child edges. The algorithm maintains a vertex queue \(Q_v\) which stores tree vertices available for expansion. 

The quantity \(g_\tau(\mathbf{x})\) denotes the cost-to-come from the root \(\mathbf{x}_k\) to a state \(\mathbf{x}\) through the current tree, where subscript \(\tau\) denotes the current tree structure \(\mathcal{T}\) for compact notation. States not currently connected to the tree are assigned \(g_\tau(\mathbf{x})=\infty\), while \(g_\tau(\mathbf{x}_k)=0\). Let \(g(\mathbf{x})\) denote the optimal cost-to-come to \(\mathbf{x}\) in the underlying sample-induced graph. Then, by construction, \(g_\tau(\mathbf{x})\ge g(\mathbf{x})\) for all \(\mathbf{x}\in \mathbf{X}_{\mathrm{s}}\cup\{\mathbf{x}_k\}\). For an edge \((\mathbf{x},\mathbf{y})\), \(c(\mathbf{x},\mathbf{y})\) denotes the true edge cost and \(\hat c(\mathbf{x},\mathbf{y})\) denotes an admissible lower-bound estimate, satisfying
\[
\hat c(\mathbf{x},\mathbf{y})\le c(\mathbf{x},\mathbf{y})\le \infty,
\]
where edges intersecting occupied or unknown cells in the occupancy map are assigned infinite cost. We will specify the cost design, which combines gas-related information with travel effort, in Section~\ref{subsec:xit_exploitation}.

The exploration module \(\Gamma_f\) encapsulates the frontier construction introduced in Section~\ref{sec:frontier-based goal regions}. In this work, \(\Gamma_f\) returns both occupancy frontiers and gas frontiers. Occupancy frontiers are obtained using standard frontier-based exploration techniques such as fast frontier detection (FFD)~\cite{Keidar2014EfficientFrontier}, whereas gas frontiers are extracted using the WGFD procedure developed later in Section~\ref{sec:gas_frontier}. For planning, the \(j\)-th frontier-derived target is represented by a nominal goal state \(\mathbf{x}_{\mathrm{goal}}^{(j)}\). For theoretical analysis, let \( \mathcal{X}_{\mathrm{goal}}^{(j)} \subset \hat{\mathcal{X}}_{\mathrm{free}} \) denote the associated continuous target neighbourhood, and let \( \mathbf{X}_{\mathrm{goal}}^{(j)} := \mathcal{X}_{\mathrm{goal}}^{(j)} \cap \mathbf{X}_{\mathrm{s}} \) denote its sample-based representation used by Algorithm~\ref{AlgX}. The frontier-derived goal domain is 
\begin{equation}
    \label{eq:goalset} 
    \mathcal{G} := \bigcup_{j=1}^{M} \mathcal{X}_{\mathrm{goal}}^{(j)}. 
\end{equation} 
Thus, \(\mathcal{X}_{\mathrm{goal}}^{(j)}\) denotes a continuous target region, whereas \(\mathbf{X}_{\mathrm{goal}}^{(j)}\) denotes the finite set of sampled states in that region.

Finally, for each goal region \(\mathbf{X}_{\mathrm{goal}}^{(j)}\), \gls{xit} maintains an incumbent terminal vertex \(\mathbf{x}_j^\star\) and an incumbent cost \(c_j\). These quantities are initialised in Alg.~\ref{AlgX} and updated during the search. Their operational role, including incumbent updates and path extraction, is described in Section~\ref{subsec:xit_algorithm}.

\subsection{Algorithm formulation}
\label{subsec:xit_algorithm}

The \gls{xit} algorithm is detailed in Alg.~\ref{AlgX}, which can be divided into an initialisation phase and a multi-direction tree expansion phase. Essentially, it performs an ordered label-updating search over the sample-induced implicit RGG to build paths toward multiple goal regions. An open-source implementation of \gls{xit} is available online\footnote{\url{https://github.com/malfazliu/XIT-Planner-Exploration-and-Exploitation-Informed-Trees-}}.

\begin{algorithm}
\label{AlgX}
\caption{\gls{xit}: graph-optimal informed tree search}

\KwIn{Sample size \(N\), state \(\mathbf{x}_k\), frontier module \(\Gamma_f\), information field \(\mathcal I\)}
\KwOut{Tree \(\mathcal{T} := (V,E)\) and paths \(\{\sigma_j^\star\}_{j=1}^{M}\)}

\tcp{Initialisation} 

\(\mathbf{X}_{\mathrm{s}}\leftarrow \mathrm{InformedSample}(N,\mathcal I)\)\;
\(\{\mathbf{X}_{\mathrm{goal}}^{(j)}\}_{j=1}^{M}
\leftarrow \mathrm{GoalSets}(\Gamma_f(),\mathbf{X}_{\mathrm{s}})\)\;

\(V\leftarrow\{\mathbf{x}_k\}\), \(E\leftarrow\emptyset\), \(Q_v\leftarrow\{\mathbf{x}_k\}\)\;
\(g_\tau(\mathbf{x}_k)\leftarrow0\),
\(g_\tau(\mathbf{x})\leftarrow\infty,\ \forall\mathbf{x}\in\mathbf{X}_{\mathrm{s}}\)\;

\For{\(j=1,\ldots,M\)}{
    \(c_j\leftarrow\infty\), \(\mathbf{x}_j^\star\leftarrow\emptyset\)\;
}

\(r_N\leftarrow
\gamma\left(\log|\mathbf{X}_{\mathrm{s}}|/|\mathbf{X}_{\mathrm{s}}|\right)^{1/n}\)\;

\(\kappa_v^j(\mathbf{v}) :=
g_\tau(\mathbf{v})+\hat h_j(\mathbf{v})\)\;

\tcp{Tree Expansion Procedure} 

\While{\(Q_v\neq\emptyset\)}{

    \(j_m\leftarrow\mathrm{NextGoal}_{\mathrm{RR}}(1,\ldots,M)\)\;
    \(\mathbf{v}_m\leftarrow\mathrm{PopBest}(Q_v,\kappa_v^{j_m})\)\;

    \(V_{\mathrm{near}}\leftarrow
    \{\mathbf{w}\in\mathbf{X}_{\mathrm{s}}
    \mid \mathbf{w}\neq\mathbf{v}_m,\ 
    \|\mathbf{w}-\mathbf{v}_m\|_2\le r_N\}\)\;

    \For{\(\mathbf{w}_m\in V_{\mathrm{near}}\)}{

        \If{\(g_\tau(\mathbf{v}_m)+\hat c(\mathbf{v}_m,\mathbf{w}_m)
        < g_\tau(\mathbf{w}_m)\)}{

            \(\bar c\leftarrow
            g_\tau(\mathbf{v}_m)+c(\mathbf{v}_m,\mathbf{w}_m)\)\;

            \If{\(\bar c<g_\tau(\mathbf{w}_m)\)}{

                \If{\(\mathbf{w}_m\notin V\)}{
                    \(V\leftarrow V\cup\{\mathbf{w}_m\}\)\;
                }
                \Else{
                    \(E\leftarrow
                    E\setminus\{(\mathrm{Parent}(\mathbf{w}_m),\mathbf{w}_m)\}\)\;
                }

                \(E\leftarrow E\cup\{(\mathbf{v}_m,\mathbf{w}_m)\}\)\;
                \(g_\tau(\mathbf{w}_m)\leftarrow\bar c\)\;
                \(Q_v\leftarrow Q_v\cup\{\mathbf{w}_m\}\)\;

                \For{\(j=1,\ldots,M\)}{
                    \If{\(\mathbf{w}_m\in\mathbf{X}_{\mathrm{goal}}^{(j)}
                    \textbf{ and } g_\tau(\mathbf{w}_m)<c_j\)}{
                        \(c_j\leftarrow g_\tau(\mathbf{w}_m)\)\;
                        \(\mathbf{x}_j^\star\leftarrow\mathbf{w}_m\)\;
                    }
                }
            }
        }
    }
}
\tcp{Path extraction}
\For{\(j=1,\ldots,M\)}{
    \(\sigma_j^\star\leftarrow
    \mathrm{ExtractPath}(\mathbf{x}_k,\mathbf{x}_j^\star,E)\)\;
}

\Return{\(\mathcal{T}=(V,E),\{\sigma_j^\star\}_{j=1}^{M}\)}\;
\end{algorithm}

\subsubsection{Initialisation (Alg.~\ref{AlgX}, Lines 1:8)}

Alg.~\ref{AlgX} begins by generating a finite sample set \(\mathbf{X}_{\mathrm{s}}\) using the informed sampling procedure described in Section~\ref{subsec:xit_exploitation}. The frontier module \(\Gamma_f\) is then invoked in the $\mathrm{GoalSets}$ function to return a sample-based goal set, for example, by selecting the \(k_n\) nearest samples in \(\mathbf{X}_{\mathrm{s}}\) to \(\mathbf{x}_{\mathrm{goal}}^{(j)}\). The search tree is initialised at the current robot state \(\mathbf{x}_k\), with \(V=\{\mathbf{x}_k\}\), \(E=\emptyset\), and \(Q_v=\{\mathbf{x}_k\}\). The root cost is set to \(g_\tau(\mathbf{x}_k)=0\), while all sampled states are assigned \(g_\tau(\mathbf{x})=\infty\). For each goal region, the incumbent cost and an incumbent terminal vertex are initialised as \(c_j=\infty\) and \(\mathbf{x}_j^\star=\emptyset\), respectively. 

To support graph search operation, the goal-specific vertex ordering criterion is defined as \[ \kappa_v^j(\mathbf{v})= g_\tau(\mathbf{v})+\hat h_j(\mathbf{v}), \] where \(\hat h_j(\mathbf{v})\) is an admissible lower bound from \(\mathbf{v}\) to the goal region \(\mathbf{X}_{\mathrm{goal}}^{(j)}\). The connection radius $r_N$ for edge search follows the standard RGG scaling \cite{Karaman2011IJRR,Gammell2015}, \[ r_N= \gamma\left( \frac{\log |\mathbf{X}_{\mathrm{s}}|} {|\mathbf{X}_{\mathrm{s}}|} \right)^{1/n}, \] where the sufficient choice of \(\gamma\) for the mixture sampling density is specified in Section~\ref{subsec:algorithm properties}.

\subsubsection{Multi-direction Tree Expansion Procedure (Alg.~\ref{AlgX}, Lines 9:27)}

The main loop for tree expansion continues while the vertex queue \(Q_v\) is non-empty. At each iteration, the function \(\mathrm{NextGoal}_{\mathrm{RR}}\) (Alg.~\ref{AlgX}, Line 10) selects a goal index \(j_m\) using a round-robin schedule over the \(M\) goal regions. This scheduler prevents the search from being driven only by the nearest or easiest frontier goal and encourages quasi-simultaneous progress toward multiple frontier-derived destinations.

For the selected goal \(j_m\), the operation \(\mathrm{PopBest}(Q,\kappa)\) removes from \(Q_v\) the vertex with the smallest value of \(\kappa_v^{j_m}\). This gives the expansion vertex \(\mathbf{v}_m\). The neighbouring samples within the RGG connection radius are then collected as \[ V_{\mathrm{near}} = \left\{ \mathbf{w}\in\mathbf{X}_{\mathrm{s}} \mid \mathbf{w}\neq\mathbf{v}_m,\; \|\mathbf{w}-\mathbf{v}_m\|_2\le r_N \right\}. \] Although \(\mathbf{v}_m\) is selected using the ordering criterion of the scheduled goal \(j_m\), all neighbouring samples \(\mathbf{w}_m \in V_{\mathrm{near}}\) are tested for possible cost-to-come improvement (Alg.~\ref{AlgX}, Lines 13:23). Therefore, the expansion of \(\mathbf{v}_m\) is not restricted to the scheduled goal. It may improve paths to any goal region through the shared tree.

Specifically, for each neighbouring sample \(\mathbf{w}_m\), XIT first evaluates the admissible lower-bound condition \[ g_\tau(\mathbf{v}_m)+\hat c(\mathbf{v}_m,\mathbf{w}_m) < g_\tau(\mathbf{w}_m). \] If this condition is not satisfied, the candidate edge cannot improve the current cost-to-come of \(\mathbf{w}_m\), and the true edge cost does not need to be evaluated. Otherwise, the true edge cost \(c(\mathbf{v}_m,\mathbf{w}_m)\) is computed. This true cost includes collision checking and the information-aware edge cost introduced in Section~\ref{subsec:xit_exploitation}. If \( g_\tau(\mathbf{v}_m) + c(\mathbf{v}_m,\mathbf{w}_m) < g_\tau(\mathbf{w}_m) \), then the candidate edge provides a lower-cost route to \(\mathbf{w}_m\).

When such an improvement is found, the search tree is either expanded or rewired (Alg.~\ref{AlgX}, Lines 16:23). If \(\mathbf{w}_m\) is not already in the tree, it is inserted into \(V\). If it is already in the tree, its previous parent edge is removed from \(E\), so that it can be rewired to the expansion vertex $\mathbf{v}_m$. In both cases, the new parent edge \((\mathbf{v}_m,\mathbf{w}_m)\) is added to \(E\), the cost-to-come label \(g_\tau(\mathbf{w}_m)\) is updated, and \(\mathbf{w}_m\) is inserted into \(Q_v\). Reinserting \(\mathbf{w}_m\) ensures that its outgoing neighbourhood can be reconsidered under the improved cost-to-come value. This update-and-reinsert step is related in spirit to incremental graph-search methods such as LPA*~\cite{KOENIG200493}, but is used here within an informed sampling-based planner over an implicit RGG. It is particularly important in XIT because the algorithm expands local neighbourhoods from queued vertices, rather than maintaining a persistent global edge queue as in BIT*.

After each successful label update, XIT checks whether \(\mathbf{w}_m\) belongs to any goal region (Alg.~\ref{AlgX}, Lines 25:27). If \( \mathbf{w}_m \in \mathbf{X}_{\mathrm{goal}}^{(j)}\) and its updated cost-to-come is lower than the current incumbent \(c_j\), then \(c_j\) and \(\mathbf{x}_j^\star\) are updated. Thus, the algorithm explicitly records the best currently known terminal vertex for each frontier-derived goal region. Once \(Q_v\) is empty, paths \(\{\sigma_j^\star\}_{j=1}^{M}\) are extracted by tracing parent edges from each incumbent terminal vertex \(\mathbf{x}_j^\star\) back to the root \(\mathbf{x}_k\). 

Algorithm~\ref{AlgX} is stated in its graph-optimal form, with the search continuing until \(Q_v = \emptyset \), which underpins the analysis in Section~\ref{subsec:algorithm properties}. For online active GDM, there is an option to use a feasibility-first variant that terminates once every goal region has a finite incumbent, i.e., \( c_j<\infty\), \(\forall j=1,\ldots,M \); alternatively, a time-budgeted variant returns the best incumbent paths available when the planning budget is reached. These variants retain \gls{xit}’s informed sampling, multi-goal scheduling, and label-updating structure while maintaining real-time responsiveness in NBT based active GDM.

\subsection{Information Field and Cost Design for Gas Exploitation}
\label{subsec:xit_exploitation}

\begin{figure}[!t]
\centering
\includegraphics[width=\columnwidth]{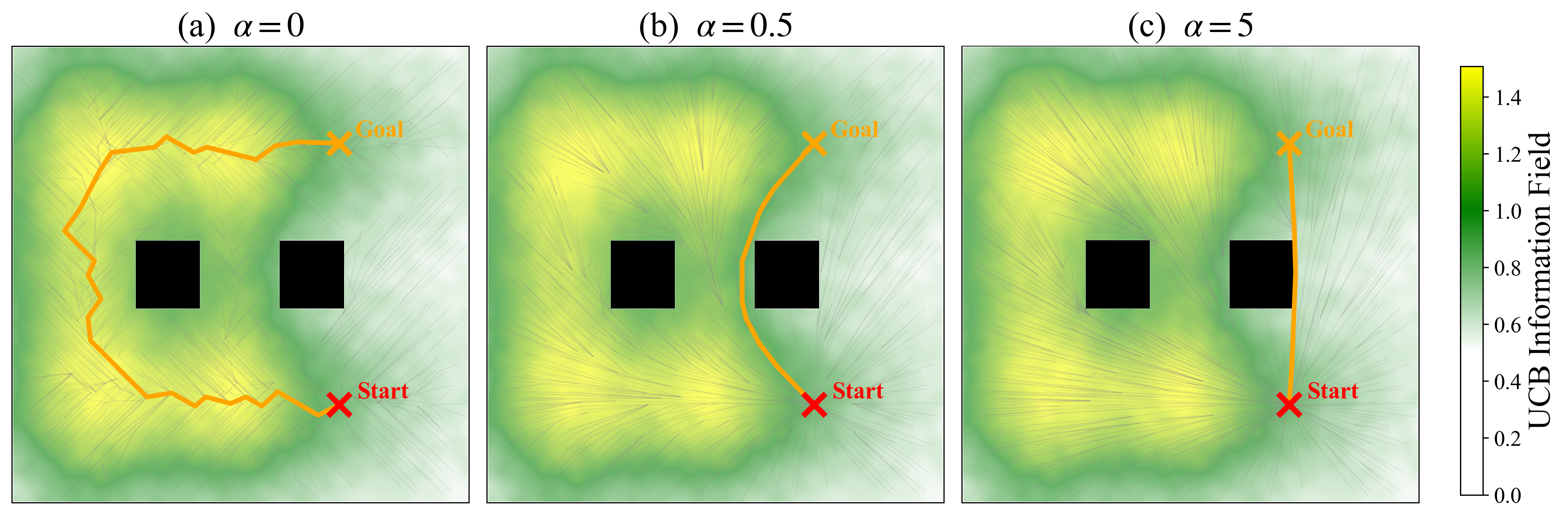}
\caption{
Effect of the distance weight \(\alpha\) on an isolated \gls{xit} trajectory to a single goal, for \(\alpha = 0\) (a), \(0.5\) (b), and \(5\) (c). All panels use the same scenario, start, goal, and sample set, so the differing routes are due to \(\alpha\) alone. A large sample budget with full graph-optimal expansion is used to show the converged routes: at \(\alpha = 0\) the cost is purely the information penalty and the path detours through high-\gls{ucb} regions, whereas larger \(\alpha\) favours the shorter geometric route. Colour shows the \gls{ucb} information field; grey lines are tree edges and orange the selected trajectory.
}
\label{fig:alpha_effect}
\end{figure}

This subsection specifies the design of the information field, informed sampling mechanism, cost formulation, and associated admissible heuristics used by Algorithm~\ref{AlgX}, while the exploration module that provides the goal regions is presented separately in Section~\ref{sec:gas_frontier}.

\paragraph*{Information field}

At planning step \(k\), the GMRF gas map $\mathcal{M}^{\mathrm{gas}}$ provides the posterior mean \(\mu(x)\) and marginal variance \(\varepsilon^2(x)\) for each cell \(x\in\hat{\mathcal X}_{\mathrm{free}}\), as introduced in Section~\ref{sec:envioronment representation}. We define the \gls{ucb}-style information field as
\begin{equation}
    \mathcal{I}(\mathbf{x}) = \tilde{\mu}(\mathbf{x}) + \beta\,\varepsilon(\mathbf{x}),
\label{eq:ucb-field}
\end{equation}
where the parameter $\beta \ge 0$ balances between exploiting high concentration and exploring regions of high uncertainty in the gas map. 
The normalised mean concentration is defined as
\begin{equation}
\tilde{\mu}(\mathbf{x}) =
\frac{\mu(\mathbf{x})}
{\max_{\mathbf{z} \in \hat{\mathcal{X}}_{\mathrm{free}}} \mu(\mathbf{z}) + \delta_\mu},
\end{equation}
where $\delta_\mu > 0$, so the denominator is strictly positive. 
Similarly, we normalise \(\mathcal{I}\) over
\(\hat{\mathcal{X}}_{\mathrm{free}}\) as
\begin{equation}
\hat{\mathcal{I}}(\mathbf{x})
= \frac{\mathcal{I}(\mathbf{x}) - \mathcal{I}_{\min}}
{\mathcal{I}_{\max} - \mathcal{I}_{\min} + \delta_I}, \\
\label{eq:normalise}
\end{equation}
where \[ \mathcal I_{\min}=\min_{\mathbf z\in\hat{\mathcal X}_{\mathrm{free}}}\mathcal I(\mathbf z), \qquad \mathcal I_{\max}=\max_{\mathbf z\in\hat{\mathcal X}_{\mathrm{free}}}\mathcal I(\mathbf z), \] and \(\delta_I>0\) avoids division by zero when the field value is spatially uniform. The bounded field \(\hat{\mathcal I}(\mathbf{x})\in[0,1)\) is used as the information field in XIT for cost or utility calculation. Note that the UCB information field is constructed from the current GMRF posterior and does not explicitly separate gas variability from uncertainty arising from limited knowledge of the distribution. GDM formulations that provide a separate model-confidence measure, such as Kernel DM approaches \cite{Lilienthal2009}, could supply or augment the uncertainty component of the information field. 

\paragraph*{Informed sampling density}

The information field \(\hat{\mathcal I}\) is also used to construct an informed sampling distribution used in the function \(\mathrm{InformedSample}\) (Alg.~\ref{AlgX}, Line 1). Let \begin{equation} q_{\mathcal I}(x)= \frac{\hat{\mathcal I}(x)} {\int_{\hat{\mathcal X}_{\mathrm{free}}}\hat{\mathcal I}(\mathbf z)\,d\mathbf z} \label{eq:importance_sampling_density} \end{equation} denote the information-biased density induced by \(\hat{\mathcal I}\), when the denominator is non-zero. Sampling directly from \(q_{\mathcal I}\) concentrates samples in regions with high predicted gas concentration or high uncertainty, but may undersample low-information regions that are still required for connectivity. Therefore, the implemented sampler uses the mixture density 
\begin{equation} 
q(x)= (1-\eta)q_{\mathcal I}(x) + \eta\frac{1}{\lambda(\hat{\mathcal X}_{\mathrm{free}})}, \qquad \eta\in(0,1), 
\label{eq:mixture_sampling_density} 
\end{equation} 
where \(\eta\) is the uniform sampling ratio and \(\lambda(\hat{\mathcal X}_{\mathrm{free}})\) denotes the measure of the current known-free space. For the theoretical analysis in Section~\ref{subsec:algorithm properties}, the sample set is modelled as \(N\) i.i.d.\ draws from the continuous mixture density \(q(x)\). In the practical implementation, \(\mathrm{InformedSample}\) instead selects \(N\) distinct known-free grid cells according to the discrete counterpart of \(q(x)\), while occupied cells, unknown cells, and the current robot cell are excluded. If \(\hat{\mathcal I}\) is identically zero, the sampler defaults to the uniform component.

\paragraph*{Edge cost function and heuristics}
The normalised information field \(\hat{\mathcal I}\) in \eqref{eq:normalise} is used not only to bias sampling, but also to guide trajectories to pass through gas-relevant regions. In particular, the \gls{ucb} based information field couples mean and variance, so the planner favours high concentration areas to reduce uncertainty where it matters, i.e., enables the exploitation of the critical gas regions. Since \gls{xit} minimises path cost, we define the information penalty as 
\begin{equation} 
P(\mathbf{x})=1-\hat{\mathcal I}(\mathbf{x}), 
\label{eq:info-pen} 
\end{equation} 
such that \(0< P(\mathbf{x})\leq1\). Thus, regions with high UCB incur a lower traversal penalty.

For a candidate local connection \((\mathbf{x},\mathbf{y})\), let \(\mathcal{S}(\mathbf{x},\mathbf{y})\) denote the set of grid cells intersected by the line segment joining \(\mathbf{x}\) and \(\mathbf{y}\). The edge cost is \begin{equation} 
c(\mathbf{x},\mathbf{y})= \sum_{\mathbf{p}\in\mathcal{S}(\mathbf{x},\mathbf{y})} P(\mathbf{p}) +\alpha\|\mathbf{y}-\mathbf{x}\|_2, 
\label{eq:edge-cost} 
\end{equation} 
where \(\alpha \geq 0\) weights travel distance against the accumulated information penalty. Smaller values of \(\alpha\) favour longer paths through more informative regions, whereas larger values increasingly favour shorter geometric paths (Fig.~\ref{fig:alpha_effect}). An edge is infeasible, with \(c(\mathbf{x},\mathbf{y})=\infty\), if any intersected cell lies outside the currently known-free space. Note that for the theoretical analysis, the formulation in \eqref{eq:edge-cost} is interpreted as the map-resolution discretisation of the continuous line cost obtained by integrating the information penalty along the segment joining \(\mathbf{x}\) and \(\mathbf{y}\).

Because the information penalty is non-negative, the distance term provides the lower-bound edge cost \begin{equation} 
\hat c(\mathbf{x},\mathbf{y}) = \alpha\|\mathbf{y}-\mathbf{x}\|_2 \leq c(\mathbf{x},\mathbf{y}), \label{eq:edge-cost-lower-bound} 
\end{equation} 
which is used in Algorithm~\ref{AlgX}, Line 14, to screen candidate edges before evaluating their full cost. Accordingly, the goal-specific heuristic is defined as 
\begin{equation} 
\hat h_j(\mathbf{x})= \alpha\,\operatorname{dist} \left(\mathbf{x},\mathcal{X}_{\mathrm{goal}}^{(j)}\right) = \alpha \min_{\mathbf{y}\in\mathcal{X}_{\mathrm{goal}}^{(j)}} \|\mathbf{y}-\mathbf{x}\|_2. 
\label{eq:goal-heuristic} 
\end{equation} 
Since it neglects the non-negative information penalty and any obstacle-induced detour, \(\hat h_j\) does not overestimate the remaining path cost.

\paragraph*{Next-best trajectory cost}
For a discrete collision-free trajectory \(\sigma=[\mathbf{x}_0,\ldots,\mathbf{x}_H]\in\Sigma\), the trajectory cost in Eq.~\eqref{eq:nbt-opt} is computed by summing the edge costs defined in Eq.~\eqref{eq:edge-cost} along the trajectory: \begin{equation} J(\sigma) = \sum_{i=0}^{H-1} c(\mathbf{x}_i,\mathbf{x}_{i+1}). \label{eq:nbt2} \end{equation} The cost is therefore additive over consecutive trajectory edges and non-decreasing under feasible trajectory extension. The \gls{xit} planner using such a cost balances travel distance against the preference for paths through informative gas regions, while the frontier-derived goal regions provide progress toward unexplored space or unresolved gas structure. 

\subsection{Algorithm Properties}
\label{subsec:algorithm properties}

This subsection establishes properties of the graph-optimal form of the \gls{xit} Algorithm~\ref{AlgX}. The analysis considers one fixed planning cycle, which means that the known-free space $\hat{\mathcal{X}}_{\mathrm{free}}$, information field $\hat{\mathcal{I}}$, and frontier-derived target regions $\mathcal{G}$, are fixed, whereas XIT performs an ordered label-updating search over the sample-induced implicit RGG~\cite{Karaman2011IJRR}. We establish its theoretical properties in three stages. First, Proposition~\ref{prop:fixed_graph_optimality} shows that, when Algorithm~\ref{AlgX} is run until \(Q_v=\emptyset\), the resulting tree contains the minimum-cost paths on the current induced graph. Second, Theorem~\ref{thm:pc} uses the non-zero support of the mixture sampler and standard RGG coverage arguments to establish probabilistic completeness. Finally, Theorem~\ref{thm:ao} combines graph optimality with standard RGG approximation results to establish almost-sure asymptotic optimality as the sample size $N$ increases.

For the completeness and optimality statements, samples are modelled as i.i.d.\ draws from the continuous mixture density \(q\) in \eqref{eq:mixture_sampling_density}. Its uniform component gives the lower density bound, i.e.,
\begin{equation} 
q(\mathbf{x})\geq q_{\min}:= \frac{\eta} {\lambda(\hat{\mathcal{X}}_{\mathrm{free}})} >0, \qquad \forall\mathbf{x}\in\hat{\mathcal{X}}_{\mathrm{free}}. 
\label{eq:minimum_sampling_density} 
\end{equation} 
Let \(\xi_n\) denote the Lebesgue measure of the unit ball in \(\mathbb{R}^n\). A conservative sufficient coefficient for the radius \(r_N\) in Algorithm~\ref{AlgX} is \begin{equation} \gamma>\gamma_q:= 2\left(1+\frac{1}{n}\right)^{1/n} \left(\frac{1}{q_{\min}\xi_n}\right)^{1/n}. \label{eq:mixture_rgg_threshold} \end{equation} Equivalently, \(\gamma_q=\eta^{-1/n}\gamma_{\mathrm{RGG}}\), where \(\gamma_{\mathrm{RGG}}\) is the corresponding uniform-sampling coefficient. This condition is conservative because it uses only the minimum sampling density contributed by the uniform component and ignores the additional probability mass assigned by the information-biased term.

\begin{proposition}[Optimality on a fixed sample-induced graph]
\label{prop:fixed_graph_optimality}
Consider a fixed sample set \(\mathbf{X}_{\mathrm{s}}\) of size \(N\). Let \(G_N=(V_N,E_N)\) denote the collision-free radius graph induced by \(\mathbf{X}_{\mathrm{s}}\), where $ V_N = \{\mathbf{x}_k\} \cup \mathbf{X}_{\mathrm{s}} $, and \(E_N\) contains all collision-free local connections between states in \(V_N\) whose Euclidean length does not exceed \(r_N\). Thus, \(G_N\) is the collision-free subgraph of the
sample-induced RGG.

For a sampled goal set \(\mathbf{X}_{\mathrm{goal}}^{(j)}\), let \( J_{G_N,j}^{\star} \) denote the minimum cost of a path in \(G_N\) from robot state \(\mathbf{x}_k\) to any reachable state in \( \mathbf{X}_{\mathrm{goal}}^{(j)} \), with \(J_{G_N,j}^{\star}=\infty\) if no such path exists.

Suppose that Algorithm~\ref{AlgX} is run until \(Q_v=\emptyset\), and that all feasible local connections from each expanded vertex are examined. Then, for each sampled goal set containing at least one
reachable state,
\begin{equation}
J(\sigma_{N,j}^{\star})
=
J_{G_N,j}^{\star}.
\label{eq:fixed_graph_optimality}
\end{equation}
\end{proposition}

\begin{IEEEproof}
For a fixed sample set, Algorithm~\ref{AlgX} acts as a heuristic-ordered label-correcting shortest-path search over \(G_N\). Each finite value of \(g_\tau(\mathbf{x})\) is the cost of a feasible graph path from \(\mathbf{x}_k\) to \(\mathbf{x}\). Whenever an edge \((\mathbf{u},\mathbf{v})\) strictly improves the label of \(\mathbf{v}\), the tree is updated and \(\mathbf{v}\) is inserted into \(Q_v\), so that its outgoing local connections are reconsidered.

The lower-bound screening step does not exclude any improving edge. 
Indeed, if
\[
g_\tau(\mathbf{u})+c(\mathbf{u},\mathbf{v})
<
g_\tau(\mathbf{v}),
\]
then, since
\(\hat c(\mathbf{u},\mathbf{v})\leq c(\mathbf{u},\mathbf{v})\),
\[
g_\tau(\mathbf{u})+\hat c(\mathbf{u},\mathbf{v})
<
g_\tau(\mathbf{v}),
\]
and the edge passes the screening test in Algorithm~\ref{AlgX}.
The subsequent exact-cost test therefore performs the required label
correction.

Since \(G_N\) is finite and every feasible edge has strictly positive cost because \(P(\mathbf{x})\geq P_{\min}>0\), exhaustive execution terminates after finitely many strict label reductions. At termination, \(Q_v=\emptyset\), so no feasible local connection can further reduce a cost-to-come label.
The standard correctness result for label-correcting shortest-path methods then gives the shortest-path labels on \(G_N\)~\cite{Bertsekas1993}. Taking the minimum over the reachable states in \( \mathbf{X}_{\mathrm{goal}}^{(j)} \) yields Eq.~\eqref{eq:fixed_graph_optimality}.
\end{IEEEproof}

\begin{theorem}[Probabilistic completeness]
\label{thm:pc}
Consider the graph-optimal form of Algorithm~\ref{AlgX}, run until \(Q_v=\emptyset\) under the assumptions of Proposition~\ref{prop:fixed_graph_optimality}. Let \(\mathcal{X}_{\mathrm{goal}}^{(j)}\) be a frontier-derived target region with non-empty interior. Suppose that there exists a \(\delta\)-clear collision-free curve from \(\mathbf{x}_k\) to a point in \(\operatorname{int}(\mathcal{X}_{\mathrm{goal}}^{(j)})\).
Let \(\mathbf{X}_{\mathrm{s}}\) contain \(N\) independent samples drawn from the mixture density in
Eq.~\eqref{eq:mixture_sampling_density}, and let 
\[
r_N=
\gamma\left(\frac{\log N}{N}\right)^{1/n},
\qquad
\gamma>\gamma_q,
\]
where \(\gamma_q\) is defined in Eq.~\eqref{eq:mixture_rgg_threshold}. Then,
\begin{equation}
\lim_{N\rightarrow\infty} \Pr\!\left( \begin{array}{c} \exists\,\sigma_{N,j}\subset\hat{\mathcal{X}}_{\mathrm{free}}: \ \sigma_{N,j}(0)=\mathbf{x}_k,\\ \sigma_{N,j}(1)\in\mathbf{X}_{\mathrm{goal}}^{(j)} \end{array} \right) =1.
\label{eq:pc_xit}
\end{equation}
\end{theorem}

\begin{IEEEproof}
The proof adapts the standard RGG path-covering construction under uniform sampling, e.g., Lemmas~52--54 of Karaman and Frazzoli~\cite{Karaman2011IJRR}, to the mixture density in Eq.~\eqref{eq:mixture_sampling_density}. Let \(\sigma^{\mathrm{f}}\) denote a \(\delta\)-clear path from \(\mathbf{x}_k\) to the interior of \(\mathcal{X}_{\mathrm{goal}}^{(j)}\). We aim to cover this path by a sequence of small balls and seek at least one sample in each ball, so that samples selected from consecutive balls form a path in the sample-induced graph.

Since \(\gamma>\gamma_q\) as defined in \eqref{eq:mixture_rgg_threshold}, a sufficiently small \(\theta_1>0\) can be chosen such that
\[ a:= q_{\min}\xi_n \left( \frac{\gamma}{2+\theta_1} \right)^n > 1+\frac{1}{n}, \]
where \(q_{\min}\) is given in Eq.~\eqref{eq:minimum_sampling_density} and \(\xi_n\) is the volume of the unit ball in \(\mathbb{R}^n\). Define
\[
\rho_N=\frac{r_N}{2+\theta_1},
\]
and cover \(\sigma^{\mathrm{f}}\) by balls \(B_i=B(\mathbf{c}_i,\rho_N)\), with consecutive centres satisfying \[ \|\mathbf{c}_{i+1}-\mathbf{c}_i\|_2 \leq \theta_1\rho_N. \] For sufficiently large \(N\), the balls lie in the \(\delta\)-clear tube around \(\sigma^{\mathrm{f}}\), and the final ball is contained in \(\mathcal{X}_{\mathrm{goal}}^{(j)}\).

Although the sampling density is information-biased, its uniform component yields the global lower bound \(q(\mathbf{x})\geq q_{\min}\). Thus, each ball has sampling probability at least \(q_{\min}\lambda(B_i)\). Since \[ \lambda(B_i) = \xi_n\rho_N^n = \xi_n \left( \frac{\gamma}{2+\theta_1} \right)^n \frac{\log N}{N}, \] the probability that a particular ball contains no sample satisfies 
\begin{align} 
\Pr\!\left( B_i\cap\mathbf{X}_{\mathrm{s}}=\emptyset \right) 
&\leq \left( 1-q_{\min}\lambda(B_i) \right)^N \nonumber\\ 
&= \left( 1-a\frac{\log N}{N} \right)^N \leq N^{-a}. 
\label{eq:pc_empty_ball_bound} 
\end{align} 
This is the mixture-sampling counterpart of the ball-occupancy bound in Lemma~52 of~\cite{Karaman2011IJRR}. Since \(a>1+1/n\), the corresponding path-covering argument implies that, with probability approaching one, every non-root ball contains at least one sample.

Let \(\mathbf{x}_i\in B_i\) and \(\mathbf{x}_{i+1}\in B_{i+1}\) be samples selected from consecutive balls. Then \begin{equation} \|\mathbf{x}_{i+1}-\mathbf{x}_i\|_2 \leq 2\rho_N+\theta_1\rho_N = r_N. \label{eq:pc_rgg_connection} \end{equation} The same construction applies to the initial connection from \(\mathbf{x}_k\). By the clearance of \(\sigma^{\mathrm{f}}\), these straight-line local connections are collision-free. Hence, with probability approaching one, \(G_N\) contains a collision-free path from \(\mathbf{x}_k\) to a sampled state in the final ball, which belongs to \(\mathbf{X}_{\mathrm{goal}}^{(j)}\). Proposition~\ref{prop:fixed_graph_optimality} then ensures that the exhaustive form of XIT returns such a path, which proves Eq.~\eqref{eq:pc_xit}.
\end{IEEEproof}

For the following almost-sure statements, let
\[
\mathbf{X}_{\mathrm{s}}^{N}
=
\{\mathbf{x}^{[1]},\ldots,\mathbf{x}^{[N]}\},
\qquad
\mathbf{X}_{\mathrm{goal}}^{(j),N}
=
\mathcal{X}_{\mathrm{goal}}^{(j)}
\cap
\mathbf{X}_{\mathrm{s}}^{N},
\]
where \(\{\mathbf{x}^{[\ell]}\}_{\ell\geq1}\) is an infinite i.i.d. sequence drawn from \(q(\mathbf{x})\). For each fixed \(N\), \(\mathbf{X}_{\mathrm{s}}^{N}\) is precisely the sample set used by Algorithm~\ref{AlgX}. This convention is used only to define convergence as \(N\rightarrow\infty\).

\begin{lemma}[Asymptotic approximation by the sample-induced graph] 
\label{lem:graph_ao}
Let \(G_N\) denote the collision-free radius graph induced by \(\mathbf{X}_{\mathrm{s}}^{N}\), and let \(J_{G_N,j}^{\star}\) be the minimum cost of a path in \(G_N\) from \(\mathbf{x}_k\) to \(\mathbf{X}_{\mathrm{goal}}^{(j),N}\). Suppose that the planning problem admits a robustly optimal trajectory \(\sigma_j^\star\) with finite cost 
\[ J_j^\star:=J(\sigma_j^\star), \] 
and that the continuous line-cost counterpart of Eq.~\eqref{eq:edge-cost} satisfies
the regularity assumptions required for asymptotic optimality. Then, under the sampling and connection-radius conditions of Theorem~\ref{thm:pc}, 
\begin{equation} 
\Pr\!\left( \lim_{N\rightarrow\infty} J_{G_N,j}^{\star} = J_j^\star \right) =1. 
\label{eq:graph_ao} 
\end{equation} \end{lemma}

\begin{IEEEproof}
The proof adapts the PRM$^{*}$ graph-approximation construction in
Appendix~C of Karaman and Frazzoli~\cite{Karaman2011IJRR}. The construction of strongly clear paths approaching \(\sigma_j^\star\) and their associated ball covers is unchanged. The distribution-dependent part is the probability that each such ball contains a sampled vertex.

For the mixture sampler used in Algorithm~\ref{AlgX}, we have
\[
\Pr(\mathbf{x}\in B)
=
\int_B q(\mathbf{x})\,d\mathbf{x}
\geq
q_{\min}\lambda(B).
\]
Thus, the ball-occupancy calculation in Lemma~52 of \cite{Karaman2011IJRR} is preserved after replacing the uniform
density by \(q_{\min}\). The condition \(\gamma>\gamma_q\) provides the corresponding summable failure-probability bound. Consequently, the connection, clearance, and path-convergence arguments corresponding to Lemmas~53--55 of \cite{Karaman2011IJRR} can apply to \(G_N\), which contains all collision-free radius-\(r_N\) connections between its sampled vertices.

Hence, there exists, almost surely, a sequence of collision-free graph paths \(\tilde{\sigma}_{N,j}\subset G_N\) such that 
\[ J(\tilde{\sigma}_{N,j}) \longrightarrow J_j^\star. \]
Since every path in \(G_N\) is feasible for the continuous planning problem, \[ J_j^\star \leq J_{G_N,j}^{\star} \leq J(\tilde{\sigma}_{N,j}). \] The result follows by squeezing. 
\end{IEEEproof}

\begin{theorem}[Almost-sure asymptotic optimality of XIT]
\label{thm:ao}
Under the assumptions of Lemma~\ref{lem:graph_ao}, the graph-optimal
form of Algorithm~\ref{AlgX}, run until \(Q_v=\emptyset\), satisfies
\begin{equation}
\Pr\!\left(
\lim_{N\rightarrow\infty}
J\!\left(\sigma_{N,j}^{\mathrm{XIT}}\right)
=
J_j^\star
\right)
=1,
\label{eq:ao_xit}
\end{equation}
where \(\sigma_{N,j}^{\mathrm{XIT}}\) is the path returned by the graph-optimal form of the \(N\)-sample XIT instance to \(\mathbf{X}_{\mathrm{goal}}^{(j),N}\).
\end{theorem}

\begin{IEEEproof}
For every fixed \(N\), Proposition~\ref{prop:fixed_graph_optimality}
gives
\[
J\!\left(\sigma_{N,j}^{\mathrm{XIT}}\right)
=
J_{G_N,j}^{\star}.
\]
The result follows directly from
Lemma~\ref{lem:graph_ao}.
\end{IEEEproof}

\begin{remark}[Scope of the theoretical guarantees]\label{rem:scope}
The PC and AO results concern the graph-optimal form of XIT for a fixed planning instance: the known-free space, information field, and frontier-derived target regions are held fixed, and the search runs until \(Q_v=\emptyset\). The AO statement is asymptotic in the sample size \(N\), which means that it is not an anytime or wall-clock optimality guarantee over the receding-horizon active GDM process. Moreover, obtaining an anytime refinement property within one planning call would require progressive sample densification and repeated graph optimisation, as in batch-refinement approaches such as BIT$^{*}$. This is intentionally outside the scope of the present XIT formulation as bounded planning latency is  more relevant in active GDM applications. The PC and AO results also assume i.i.d.\ samples from the continuous mixture density \(q\) in Eq.~\eqref{eq:mixture_sampling_density}, whereas the implementation selects \(N\) distinct known-free grid cells from its discrete counterpart. These guarantees therefore apply to the idealised sampling formulation and not directly to the finite-grid implementation.
\end{remark}

The feasibility-first and time-budgeted variants defined in Section~\ref{subsec:xit_algorithm} need not be optimal on a fixed sample-induced graph. This trade-off is appropriate for active GDM: after the selected trajectory is executed and the resulting measurements are assimilated, the environment representation and frontier-derived goal regions are updated, yielding a new NBT problem. Promptly obtaining informative feasible trajectories is therefore generally more useful than exhaustively refining a planning snapshot that will soon be replaced. Fig.~\ref{fig:scenarios_traj} provides illustrative examples of the feasibility-first trajectories returned by \gls{xit}.
\begin{figure}[htb]
    \centering
    \includegraphics[width=0.95\columnwidth]{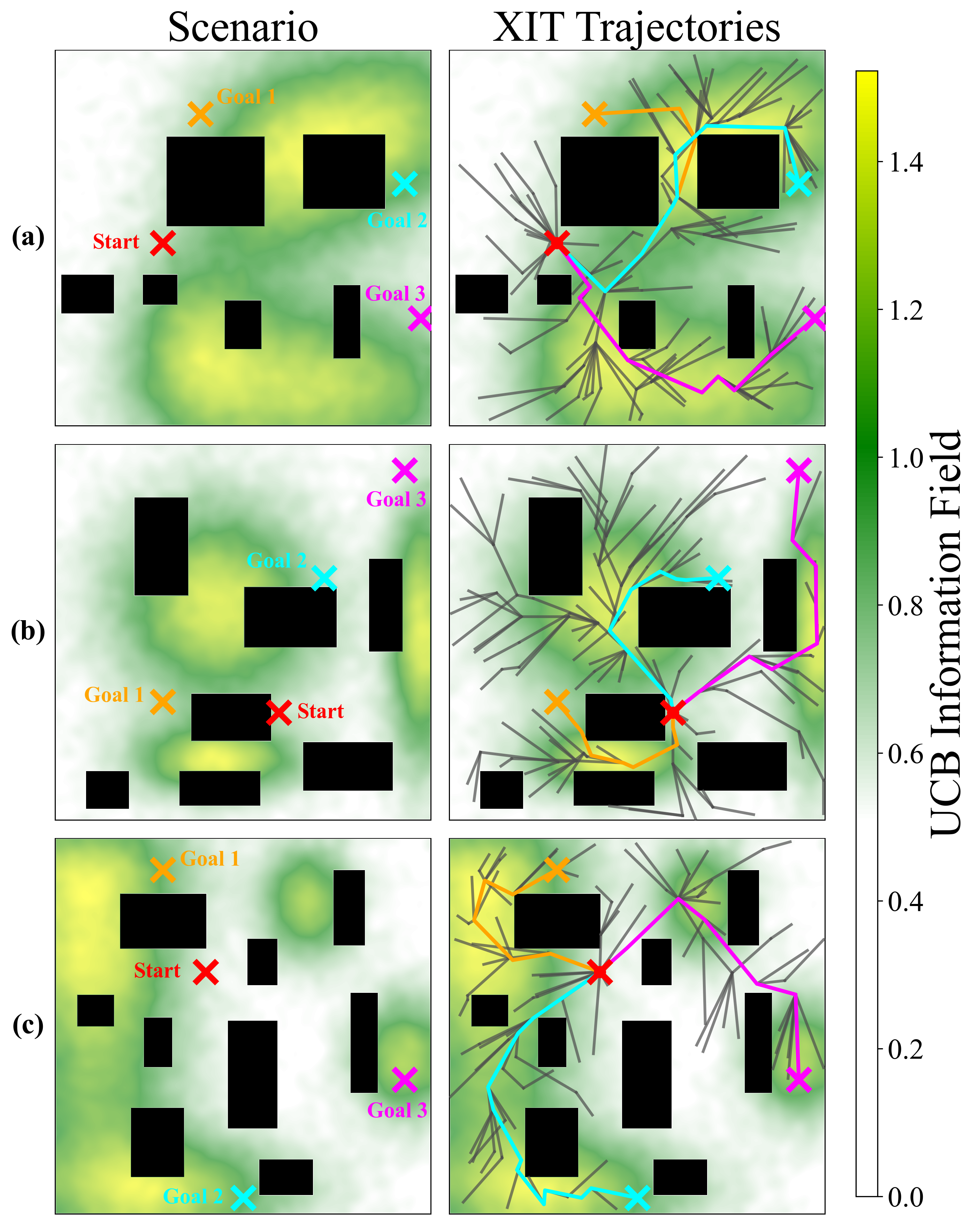}
    \caption{
        Example trajectories produced by \gls{xit} in three synthetic gas-polluted scenarios (a)--(c). The left column shows each scenario, including the \gls{ucb} information field, obstacles, start state, and frontier-derived goals. The right column shows the single \gls{xit} tree (grey), grown concurrently toward all goals, together with the trajectory recovered to each goal. The trees use \(N=150\) informed samples and \(\alpha=0\), so the explicit Euclidean-distance term is absent and the selected trajectories preferentially pass through high-\gls{ucb} regions while progressing toward their respective goals.
    }
    \label{fig:scenarios_traj}
\end{figure}

\section{Gas Frontier Detection}
\label{sec:gas_frontier}
As introduced in Section~\ref{sec:frontier-based goal regions}, gas frontiers $\mathcal{F}_k^{\mathrm{gas}}$ are intended to identify the boundary of high concentration regions that remain adjacent to unobserved gas states. They complement conventional occupancy frontiers $\mathcal{F}_k^{\mathrm{occ}}$ by steering the robot toward plume boundaries rather than only the edge of the known geometric map. Figure~\ref{fig:gas_frontier} illustrates this concept in isolation of $\mathcal{F}_k^{\mathrm{occ}}$, showing both the cell-level depiction and the resulting gas exploration behaviour as gas frontiers guide the robot to probe unresolved plume boundaries. We present the \acrfull{wgfd} algorithm to detect them. The proposed gas edge detection algorithm is based on \gls{wfd}, an iterative method introduced in \cite{Keidar2014EfficientFrontier} that conducts a graph search over known free-space regions of an occupancy grid map to identify frontiers. Frontier algorithms relying on raw sensor readings, such as \gls{ffd} \cite{Keidar2014EfficientFrontier}, are not suitable for gas frontier detection, since the \textit{in situ} gas sensor utilised provides point measurements that cannot directly infer gas frontiers but can be processed into a gas distribution map. \gls{wgfd} applies the same wavefront principles as \gls{wfd} to identify unknown cells that lie adjacent to observed regions with high gas concentration.

\subsection{Formal Definition of Gas Frontiers}

We first formalise the notion of a gas frontier cell at a fixed planning step $k$ and, following the convention in Section~\ref{sec:notation}, drop the index $k$ on all quantities where this causes no ambiguity. Recall that the occupancy grid $\mathcal{M}^{\mathrm{occ}}$ is defined on the lattice $\mathcal{C}$, and that the known free subset $\hat{\mathcal{X}}_{\mathrm{free}}$ in~\eqref{eq:known-free} corresponds to the union of cells whose occupancy probability is below the free threshold $\tau_{\mathrm{free}}$. At the level of individual cells, we define the set of free cells as
\begin{equation}
    \mathcal{C}^{\mathrm{free}} := \{\, c \in \mathcal{C} \mid p(c) \le \tau_{\mathrm{free}} \,\}.
    \label{eq:c_free}
\end{equation}

The gas distribution map $\mathcal{M}^{\mathrm{gas}}$ provides the posterior mean $\mu(c)$ and variance $\varepsilon^2(c)$ for each free cell $c \in \mathcal{C}^{\mathrm{free}}$, obtained from the GMRF inference described in Section~\ref{sec:envioronment representation}. Let $\mathcal{C}^{\mathrm{gas,obs}} \subseteq \mathcal{C}^{\mathrm{free}}$ denote the subset of cells whose gas state has been sufficiently constrained by measurements, and let
\[
\mathcal{C}^{\mathrm{gas,unk}} := \mathcal{C}^{\mathrm{free}} \setminus \mathcal{C}^{\mathrm{gas,obs}}
\]
denote the free cells whose gas state remains unobserved or highly uncertain. The precise definition of $\mathcal{C}^{\mathrm{gas,obs}}$ is determined by the inference procedure; in this work, a cell belongs to $\mathcal{C}^{\mathrm{gas,obs}}$ once its marginal variance $\varepsilon^2(c)$ has been reduced below its prior level, i.e., the variance the field assigns in the absence of measurements, and to $\mathcal{C}^{\mathrm{gas,unk}}$ otherwise.

\begin{figure}[t]
    \centering
    \includegraphics[width=0.85\columnwidth]{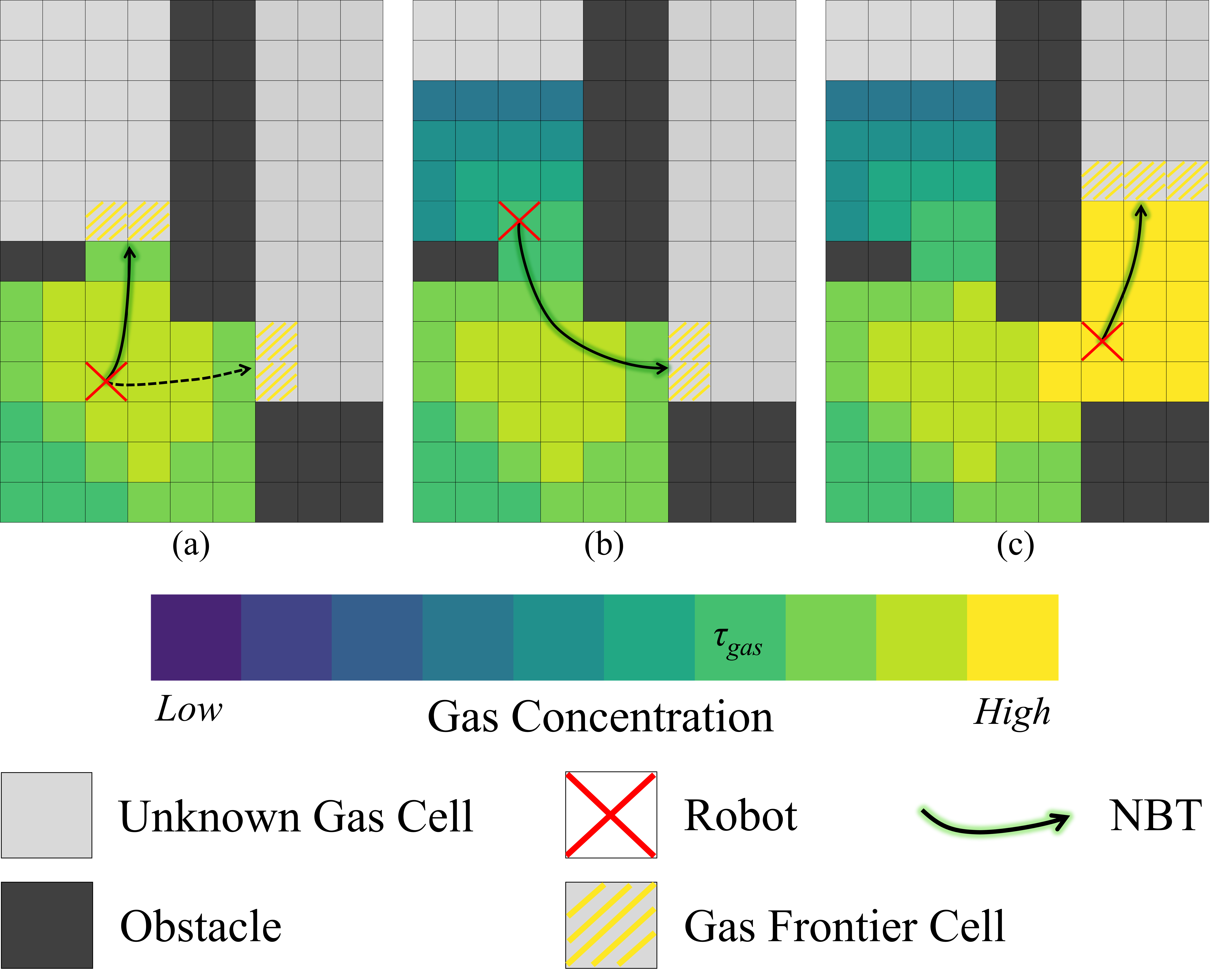}
\caption{
Cell-level depiction of gas frontiers and their role in path planning within an active \gls{gdm} task over three sequential time steps. 
(a) The robot enters from the bottom left and finds itself within a critical gas region above $\tau_{\text{gas}}$ based on its updated belief. Two gas frontiers are identified, and one is selected as the \gls{nbt}. 
(b) Upon reaching the gas frontier, the updated belief reveals that the area is no longer a priority as no new gas frontiers emerged based on the estimate, suggesting a less concentrated region. It now targets the other previously identified gas frontier from (a), which has remained unresolved. 
(c) Resolving the second gas frontier reveals a high-concentration zone that generates a new gas frontier, which can be explored next to continue the adaptive plume probing process.
}
    \label{fig:gas_frontier}
\end{figure}

To identify gas frontiers we require a gas concentration threshold that selects cells belonging to significant gas regions. We define a dynamic threshold $\tau_{\text{gas}}$ as
\begin{equation}
    \tau_{\text{gas}} := \max\!\bigl(\, q_p,\ \tau_{\text{gas,min}} \,\bigr),
    \label{eq:gamma_def}
\end{equation}
where $q_p$ is the $p$th percentile of the posterior mean values $\mu(c)$ over $\mathcal{C}^{\mathrm{gas,obs}}$, and $\tau_{\text{gas,min}} > 0$ is a user specified lower bound that prevents overly permissive detection in near zero gas conditions. A free cell $c \in \mathcal{C}^{\mathrm{gas,obs}}$ is said to be gas critical if
\begin{equation}
    \mu(c) \ge \tau_{\text{gas}},
\end{equation}
and the set of gas critical cells is denoted
\begin{equation}
    \mathcal{C}^{\mathrm{crit}} := \{\, c \in \mathcal{C}^{\mathrm{gas,obs}} \mid \mu(c) \ge \tau_{\text{gas}} \,\}.
\end{equation}
We adopt a neighbourhood operator $\mathcal{N}(c) \subset \mathcal{C}$, which returns the set of cells adjacent to $c$ on the lattice. A cell $c \in \mathcal{C}^{\mathrm{gas,unk}}$ is classified as a gas frontier cell if it lies in the unknown gas region yet is adjacent to at least one gas critical cell, that is,
\begin{equation}
    c \in \mathcal{C}^{\mathrm{gas,unk}} \text{ is a gas frontier cell if } 
    \mathcal{N}(c) \cap \mathcal{C}^{\mathrm{crit}} \neq \emptyset.
    \label{eq:gf_cell}
\end{equation}
The set of all gas frontier cells is written
\begin{equation}
    \mathcal{C}^{\mathrm{gf}} := \{\, c \in \mathcal{C}^{\mathrm{gas,unk}} \mid \mathcal{N}(c) \cap \mathcal{C}^{\mathrm{crit}} \neq \emptyset \,\}.
\end{equation}

A gas frontier $f^{\mathrm{gas}}$ is defined as a maximal connected subset of $\mathcal{C}^{\mathrm{gf}}$ under the neighbourhood relation, meaning that any two cells in $f^{\mathrm{gas}}$ can be joined by a path of neighbouring gas frontier cells and no strict superset of $f^{\mathrm{gas}}$ has this property. The collection of all gas frontiers at the current planning step is written
\begin{equation}
    \mathcal{F}^{\mathrm{gas}} := \{\, f^{\mathrm{gas}}_1, f^{\mathrm{gas}}_2, \ldots, f^{\mathrm{gas}}_M \,\},
\end{equation}
which corresponds to the set $\mathcal{F}_k^{\mathrm{gas}}$ introduced in Section~\ref{sec:frontier-based goal regions} once the index $k$ is reinstated. Each $f^{\mathrm{gas}}_m \in \mathcal{F}^{\mathrm{gas}}$ captures the boundary of a high concentration gas region that is adjacent to unknown gas states (Fig.~\ref{fig:gas_frontier}). These clusters provide the gas frontier component of the frontier based goal set $\mathcal{G}_k$ that drives plume aware exploration in the active \gls{gdm} framework.

\subsection{Wavefront Gas Frontier Detection}

The algorithm begins by initialising the gas frontier set $\mathcal{F}^{\mathrm{gas}}$ as empty and identifying the cell $c_r \in \mathcal{C}^{\mathrm{free}}$ that contains the robot's current pose $\mathbf{x}_k$ (Alg.~\ref{alg:WGFD}, Lines 1:2). The graph search proceeds only if the robot is located within a gas critical region. If the posterior mean at $c_r$ does not reach the threshold, $\mu(c_r) < \tau_{\text{gas}}$, then no local plume boundary is considered relevant and the algorithm returns the empty frontier set (Alg.~\ref{alg:WGFD}, Lines 3:4). Otherwise, the gas frontier flags $\phi_{\mathrm{gf}}$ and visited flags $\phi_v$ are initialised for all cells, the robot cell $c_r$ is marked visited, and the outer \gls{bfs} queue $Q$ is initialised with $c_r$ (Alg.~\ref{alg:WGFD}, Lines 5:7).
\begin{algorithm}
\caption{\gls{wgfd}$(\mathcal{M}^{\mathrm{gas}}, \mathcal{M}^{\mathrm{occ}}, \mathbf{x}_k, \tau_{\text{gas}})$}
\label{alg:WGFD}

$\mathcal{F}^{\mathrm{gas}} \gets \emptyset$\;
$c_r \gets c(\mathbf{x}_k)$\;
\If{$\mu(c_r) < \tau_{\mathrm{gas}}$}{
    \Return{$\mathcal{F}^{\mathrm{gas}}$}
}

{
$\phi_{\mathrm{gf}}(c),\, \phi_v(c) \gets \text{False}$ for all $c \in \mathcal{C}$\;
}

$\phi_v(c_r) \gets \text{True}$\;
Initialise BFS queue $Q$ with $c_r$\;

\While{$Q \neq \emptyset$}{
    $c \gets \text{dequeue}(Q)$\;
    \ForEach{$c_n \in \mathcal{N}(c)$}{
        \If{$\mathrm{isNewGasFrontierCell}(c_n)$}{
            $\phi_{\mathrm{gf}}(c_n) \gets \text{True}$\;
            $f^{\mathrm{gas}} \gets \{c_n\}$\;
            Initialise BFS queue $Q_f$ with $c_n$\;

            \While{$Q_f \neq \emptyset$}{
                $c' \gets \text{dequeue}(Q_f)$\;
                \ForEach{$c'_n \in \mathcal{N}(c')$}{
                    \If{$\mathrm{isNewGasFrontierCell}(c'_n)$}{
                        $\phi_{\mathrm{gf}}(c'_n) \gets \text{True}$\;
                        $f^{\mathrm{gas}} \gets f^{\mathrm{gas}} \cup \{c'_n\}$\;
                        $\text{enqueue}(c'_n, Q_f)$\;
                    }
                }
            }

            \If{$|f^{\mathrm{gas}}| \ge \texttt{min\_frontier\_size}$}{
                $\mathcal{F}^{\mathrm{gas}} \gets \mathcal{F}^{\mathrm{gas}} \cup \{f^{\mathrm{gas}}\}$\;
            }
        }
        \ElseIf{$c_n \in \mathcal{C}^{\mathrm{crit}}$ and $\phi_v(c_n) = \text{False}$}{
            $\phi_v(c_n) \gets \text{True}$\;
            $\text{enqueue}(c_n, Q)$\;
        }
    }
}

\Return{$\mathcal{F}^{\mathrm{gas}}$}
\end{algorithm}

The outer \gls{bfs} propagates only through the connected component of $\mathcal{C}^{\mathrm{crit}}$ that contains the robot's current cell (Alg.~\ref{alg:WGFD}, Lines 8:26). Although $\mathcal{C}^{\mathrm{crit}}$ denotes the set of all gas-critical cells in the map, the wavefront cannot cross subcritical regions, so only the local gas region is explored. This locality is deliberate: disconnected gas regions are discovered by occupancy-frontier exploration and acquire their own gas frontiers once entered, while already-resolved regions are not re-scanned at every planning cycle.
At each iteration, the next cell $c$ is dequeued from $Q$ and its neighbours $c_n \in \mathcal{N}(c)$ are examined (Alg.~\ref{alg:WGFD}, Line 9). If a neighbour satisfies the new gas frontier predicate in Alg.~\ref{alg:isNewGasFrontierCell}, it is immediately flagged as a gas frontier cell, used to initialise a new frontier set $f^{\mathrm{gas}}$, and enqueued into an inner \gls{bfs} queue $Q_f$ (Alg.~\ref{alg:WGFD}, Lines 11:14). The inner \gls{bfs} then groups all connected gas frontier cells into the same frontier (Alg.~\ref{alg:WGFD}, Lines 15:21): while $Q_f$ is not empty, the current cell $c'$ is dequeued, its neighbours $c'_n$ are checked, and any neighbour that again qualifies as a new gas frontier cell is flagged, added to $f^{\mathrm{gas}}$, and enqueued into $Q_f$. This inner loop terminates once no further neighbouring gas frontier cells can be found, yielding a maximal connected subset $f^{\mathrm{gas}} \subset \mathcal{C}^{\mathrm{gf}}$ in the sense of the formal definition. Frontiers whose cardinality falls below a minimum size threshold are discarded, filtering out small, noise induced structures (Alg.~\ref{alg:WGFD}, Lines 22–23).

If a neighbour $c_n$ is not a gas frontier cell but belongs to the gas critical set $\mathcal{C}^{\mathrm{crit}}$ being explored and has not yet been visited, it is marked visited and enqueued into the outer \gls{bfs} queue $Q$ (Alg.~\ref{alg:WGFD}, Lines 24:26). In this way, the outer \gls{bfs} explores the interior of the gas region above the threshold, while the inner \gls{bfs} extracts the plume boundary where this region borders unknown gas states.

\begin{algorithm}
\caption{\textnormal{isNewGasFrontierCell}$(c)$}
\label{alg:isNewGasFrontierCell}

\If{$c \notin \mathcal{C}^{\mathrm{gas,unk}}$}{
    \Return{\text{False}} \tcp*{Gas state is not unknown}
}
\If{$\phi_{\mathrm{gf}}(c) = \text{True}$}{
    \Return{\text{False}} \tcp*{Already flagged as gas frontier}
}
\If{$\mathcal{N}(c) \cap \mathcal{C}^{\mathrm{crit}} = \emptyset$}{
    \Return{\text{False}} \tcp*{No neighbouring gas critical cell}
}
\Return{\text{True}}
\end{algorithm}

The helper function in Alg.~\ref{alg:isNewGasFrontierCell} implements the cell level definition in Eq.~\eqref{eq:gf_cell}. A candidate cell $c$ must lie in the gas unknown set $\mathcal{C}^{\mathrm{gas,unk}}$, must not already have been assigned to an existing frontier, and must have at least one neighbouring cell in the gas critical set $\mathcal{C}^{\mathrm{crit}}$. Cells that violate any of these conditions are rejected. When the outer \gls{bfs} queue $Q$ becomes empty, all reachable gas critical regions have been processed and the procedure returns $\mathcal{F}^{\mathrm{gas}}$.

\section{\texorpdfstring{XIT Characterisation and Ablation Studies}{XIT Characterisation and Ablation Studies}}

\begin{figure*}[!t]
    \centering
    \hspace*{-3.2mm}\includegraphics[width=0.86\textwidth]{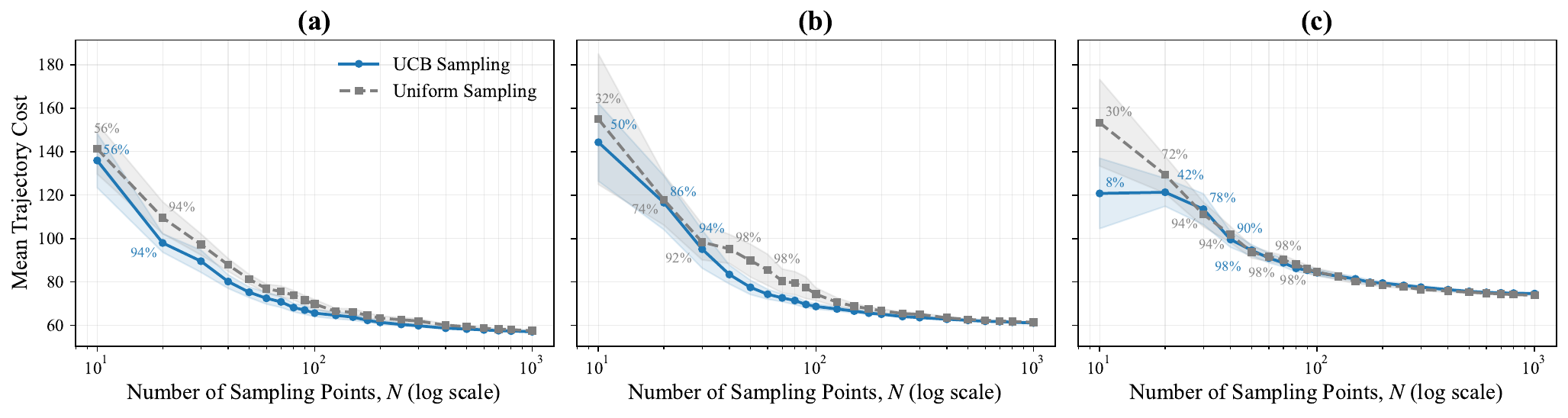}\hspace*{3.2mm}%
    \caption{
        Goal-averaged graph-optimal trajectory cost for the three scenarios in Fig.~\ref{fig:scenarios_traj}, against sample budget \(N\). \gls{ucb}-guided mixture sampling is compared with uniform sampling using \(\alpha=0\). Curves show the mean over the runs (of 50 random seeds) in which all three goal regions were reached; shaded bands indicate 95\% confidence intervals, and labels give the success rate where below 100\%. Both strategies approach the same low-cost regime as \(N\) increases, whereas \gls{ucb}-guided sampling provides an advantage when the sample size is limited, most visibly in (a) and (b).
    }
    \label{fig:convergence}
\end{figure*}

\begin{figure*}[!t]
    \centering
    \hspace*{0.3mm}\includegraphics[width=0.929\textwidth]{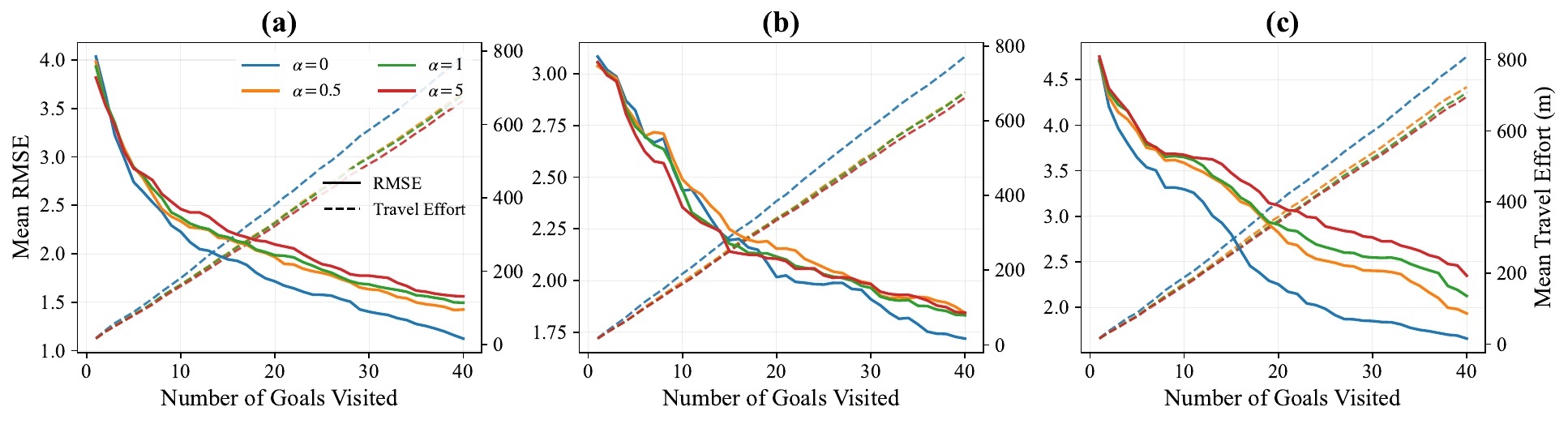}\hspace*{-0.3mm}%
    \caption{
        Mean gas-map RMSE (solid, left axis) and mean travel effort (dashed, right axis) versus the number of goals visited (one goal per planning cycle), over 50 seeded missions in the three scenarios of Fig.~\ref{fig:scenarios_traj} ((a)--(c)). Smaller \(\alpha\) invests greater travel effort per goal and yields consistently lower mapping error.
    }
    \label{fig:alpha_ablation}
\end{figure*}

\begin{figure*}[!t]
    \centering
    \hspace*{-6.7mm}\includegraphics[width=0.9\textwidth]{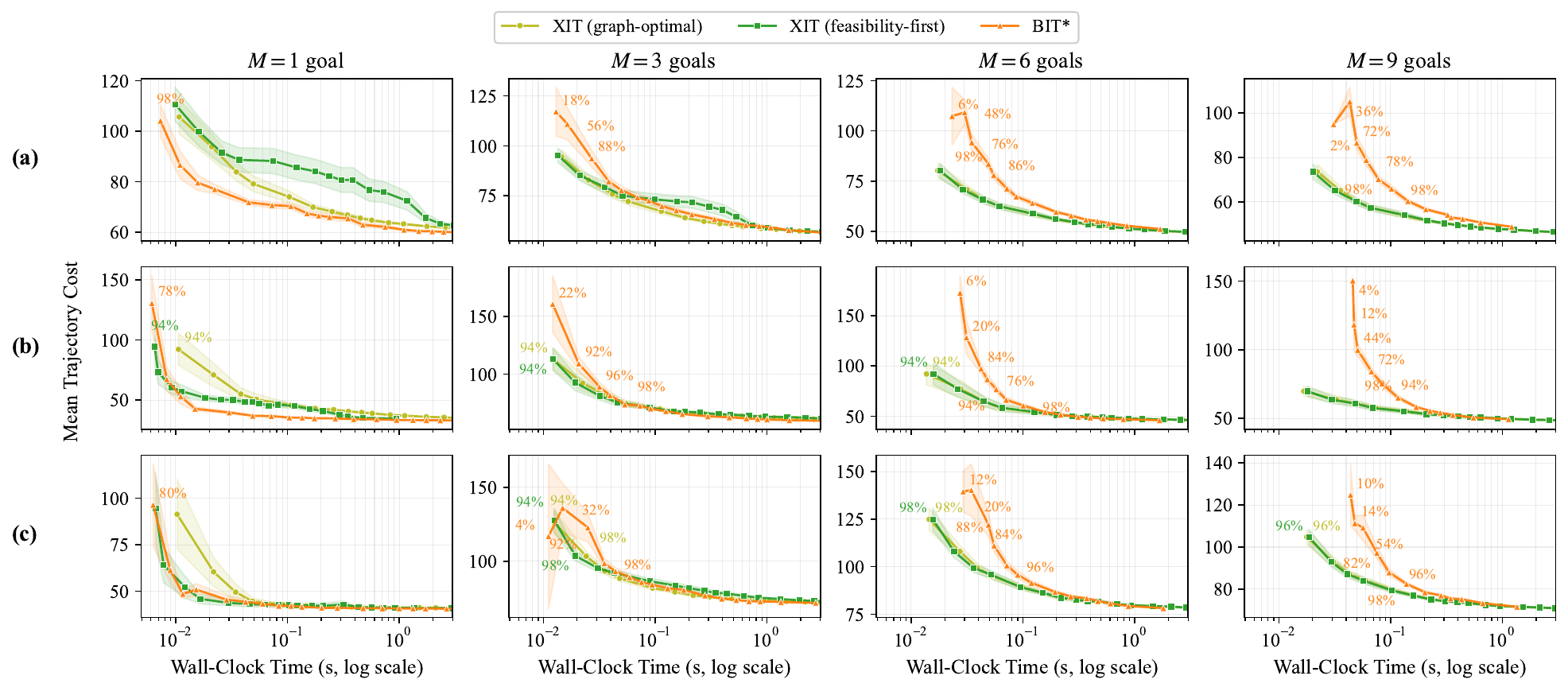}\hspace*{6.7mm}%
    \caption{
        Goal-averaged trajectory cost over runs reaching all \(M\) goals (50 seeds; band: 95\% confidence interval) versus wall-clock planning time, for \gls{xit} (graph-optimal and feasibility-first) and a sequential \gls{bit*} baseline (\(\alpha=0\)). Rows (a)--(c) are the three scenarios of Fig.~\ref{fig:scenarios_traj}; columns are \(M=1,3,6,9\) goals. Each marker is a different sample budget \(N\), not the progress of a single run. Labels give the fraction of runs reaching all \(M\) goals where below 100\%.
    }
    \label{fig:cost_vs_time}
\end{figure*}

This section empirically characterises the principal design choices of \gls{xit} under fixed planning snapshots and in the closed \gls{gdm} loop. The studies isolate the effect of \gls{ucb}-guided sampling, examine the closed-loop information--distance trade-off controlled by \(\alpha\), and quantify the computational benefit of the shared multi-goal tree structure in comparison with BIT$^*$. These controlled studies complement the end-to-end evaluation of the complete receding-horizon active-GDM framework in Section~\ref{sec:high-fidelity simulation}.

\subsection{\texorpdfstring{Effect of UCB-Guided Sampling}{Effect of UCB-Guided Sampling}}

We first examine the finite-sample effect of the \gls{ucb}-guided mixture sampler using the three fixed planning scenarios illustrated in Fig.~\ref{fig:scenarios_traj}. To isolate the sampling mechanism, we use the graph-optimal form of \gls{xit} and set \(\alpha=0\), such that the explicit Euclidean-distance term is absent from the objective. For each sample budget \(N\), Fig.~\ref{fig:convergence} reports the goal-averaged trajectory cost over 50 seeded runs.

Both sampling strategies approach the same low-cost regime as \(N\) increases, consistent with the asymptotic analysis in Section~\ref{subsec:algorithm properties}. The benefit of the \gls{ucb}-guided mixture is therefore principally at smaller sample budgets, which are the budgets relevant to real-time receding-horizon operation. The magnitude of this advantage is scenario-dependent: it is greatest where informative regions lie along the routes between the start and the goals, as in (a) and (b), so that early samples in these regions directly improve the attainable trajectories, and smaller where they do not, as in (c). %

These results show that \gls{ucb}-guided sampling improves finite-sample efficiency while preserving global coverage through its uniform component (Eq.~\eqref{eq:mixture_sampling_density}). The information-biased component preferentially allocates samples to regions that are valuable for gas exploitation and uncertainty reduction.

\subsection{\texorpdfstring{Alpha Ablation}{Alpha Ablation}}
\label{subsec:alpha_ablation}

We next quantify the effect of the distance weight \(\alpha\) in Eq.~\eqref{eq:edge-cost} on closed-loop \gls{gdm} accuracy. The receding-horizon loop is reproduced in a controlled Monte Carlo setting with a static ground-truth plume: at each planning cycle, the \gls{gmrf} posterior yields the \gls{ucb} information field, \gls{xit} (\(N{=}300\), feasibility-first form) plans to six goals drawn uniformly over the free space, and the selected \gls{nbt} (Eq.~\eqref{eq:nbt2}) is executed at 1.25\,m/s with gas sampled at 2\,Hz, after which the \gls{gmrf} is rebuilt from all measurements. Missions comprise 40 planning cycles (one goal visited per cycle) in each of the three scenarios of Fig.~\ref{fig:scenarios_traj}, with the \gls{gmrf} using a 1.5\,m factor graph over the 0.2\,m occupancy grid (\(\sigma_\zeta{=}0.05\), \(\sigma_s{=}1.0\), \(\sigma_r{=}0.8\), \(\sigma_d{=}150\), in the notation of \cite{Rhodes2023STRUC}). Every \(\alpha\) value is evaluated over the same 50 seeded goal sequences, so performance differences are attributable to \(\alpha\) alone.

Mapping accuracy is measured against the ground truth over the critical gas region \(\mathcal{X}_{\mathrm{crit}} = \{\mathbf{x} \in \mathcal{X}_{\mathrm{free}} \mid \mu_{\mathrm{GT}}(\mathbf{x}) > z_{\text{thresh}}\}\), \(z_{\text{thresh}} = 2.5\)\,ppm, as
\begin{equation}
\mathrm{RMSE}_k =
\sqrt{
\frac{1}{|\mathcal{X}_{\mathrm{crit}}|}
\sum_{\mathbf{x} \in \mathcal{X}_{\mathrm{crit}}}
\left[
\hat{\mu}_k(\mathbf{x}) -
\mu_{\mathrm{GT}}(\mathbf{x})
\right]^2
}, \label{eq:rmse}
\end{equation}
where \(\hat{\mu}_k(\mathbf{x})\) is the \gls{gmrf} posterior mean after planning cycle \(k\) and \(\mu_{\mathrm{GT}}(\mathbf{x})\) the ground-truth concentration. The corresponding map uncertainty is the total differential entropy
\begin{equation}
H_k = \sum_{\mathbf{x} \in \mathcal{X}_{\mathrm{crit}}} \tfrac{1}{2}\log\big(2\pi e\, \varepsilon^2_k(\mathbf{x})\big), \label{eq:entropy}
\end{equation}
where \(\varepsilon^2_k(\mathbf{x})\) is the marginal posterior variance at \(\mathbf{x}\). Fig.~\ref{fig:alpha_ablation} reports the mean \(\mathrm{RMSE}_k\) over the 50 missions, together with the mean travel effort, i.e., the cumulative distance travelled. Across all three scenarios, \(\alpha=0\) attains the lowest RMSE at the cost of the greatest travel effort, while larger \(\alpha\) progressively trades mapping accuracy for shorter routes.

\subsection{\texorpdfstring{Computational Behaviour with Multiple Goals}{Computational Behaviour with Multiple Goals}}
\label{subsec:bitstar_comparison}

\gls{xit} grows a single informed tree toward all \(M\) goals at once, rather than planning to each separately. We compare it against a sequential \gls{bit*} baseline that plans to each goal in turn (in three sample batches) with \(N/M\) samples per goal, matching \gls{xit}'s total budget \(N\). For fairness, the baseline uses the same information-penalty cost and draws its initial batch from the same \gls{ucb}-informed density; its subsequent batches use \gls{bit*}'s native informed (ellipse) sampling, which we leave unchanged. The comparison uses \(\alpha=0\), the pure information-penalty setting used in Fig.~\ref{fig:convergence} and identified in Section~\ref{subsec:alpha_ablation} as yielding the most accurate closed-loop \gls{gdm}, and is repeated on the three scenarios of Fig.~\ref{fig:scenarios_traj}. Starting from a single goal at \(M=1\), further goals are added and dispersed across the workspace as \(M\) increases. Plotting cost against wall-clock time makes the comparison independent of each planner's per-sample cost. Timings are on an Intel Core i9-11900KF at 3.5\,GHz.

At \(M=1\) the dedicated single-goal baseline is, as expected, slightly stronger, reaching low-cost solutions sooner, since \gls{bit*} is designed precisely for this single-goal setting; as \(M\) grows \gls{xit} reaches low-cost solutions with markedly less computation (Fig.~\ref{fig:cost_vs_time}), and the gap widens because \gls{xit} amortises one tree across all goals while the sequential baseline repeats the search per goal. This multi-goal amortisation, together with the information-biased sampling, is what separates \gls{xit} from a single-goal \gls{bit*} with a modified cost. Of the two \gls{xit} modes shown, the feasibility-first form (used online) terminates as soon as every goal is connected, and therefore typically returns its trajectories with less computation than the exhaustive graph-optimal form. This early termination sacrifices no reachability: both modes report the same success rate, since whether a goal is reachable is a property of the sampled graph, not of the termination rule. It also suits online use, as per-cycle planning effort adapts to the current map and goal set at a fixed \(N\), rather than requiring \(N\) to be re-tuned to each environment to meet a time budget. As the goals become numerous, the problem approaches one of coverage: connecting every dispersed goal requires the tree to span most of the free space, so the feasibility-first search comes close to the full expansion performed by the graph-optimal form, and the two modes converge in both cost and computation.

\section{High-Fidelity Simulation Study}
\label{sec:high-fidelity simulation}
While no established benchmark exists for active \gls{gdm} in unknown, cluttered environments, this section evaluates the proposed \gls{xit} planner within our active \gls{gdm} framework (Fig.~\ref{fig:overview}). We perform high-fidelity Monte Carlo simulation in \gls{ros}, with the GADEN simulator~\cite{Monroy2017} providing realistic gas dispersion in two large-scale scenarios: a cluttered industrial unit layout and a warehouse. Performance is reported relative to an RRT* frontier-exploration baseline, and we then examine the influence of \gls{xit}'s key design choices, namely goal-selection policy and cost formulation, using representative limiting cases of the proposed \gls{ucb}-distance cost.

\subsection{Simulation Setup}
\label{sec:sim_setup}

Two simulation scenarios are considered. The first comprises a cluttered industrial-style environment with storage containers and partitioned rooms. The second is a larger warehouse with aisles of storage racks, as shown in Fig.~\ref{fig:hifi_scenario_setup}. In each scenario, three acetone sources were placed at distinct locations with varying release rates detailed in Table~\ref{tab:sim_parameters}. The filament growth rate in GADEN was set to 15 cm$^2$/s for all sources. Both scenarios use steady airflow fields obtained from CFD simulation of the enclosed spaces, with typical speeds of 0.3 m/s (industrial) and 0.1 m/s (warehouse), the latter producing broader, slowly advected plumes. Although GADEN models source releases as unsteady processes, these wind fields impose a consistent advection regime and ensure approximately time-invariant plume structures, which is essential for fair comparison across trials. %

\begin{figure}[!b]
    \centering
    \includegraphics[width=0.92\columnwidth]{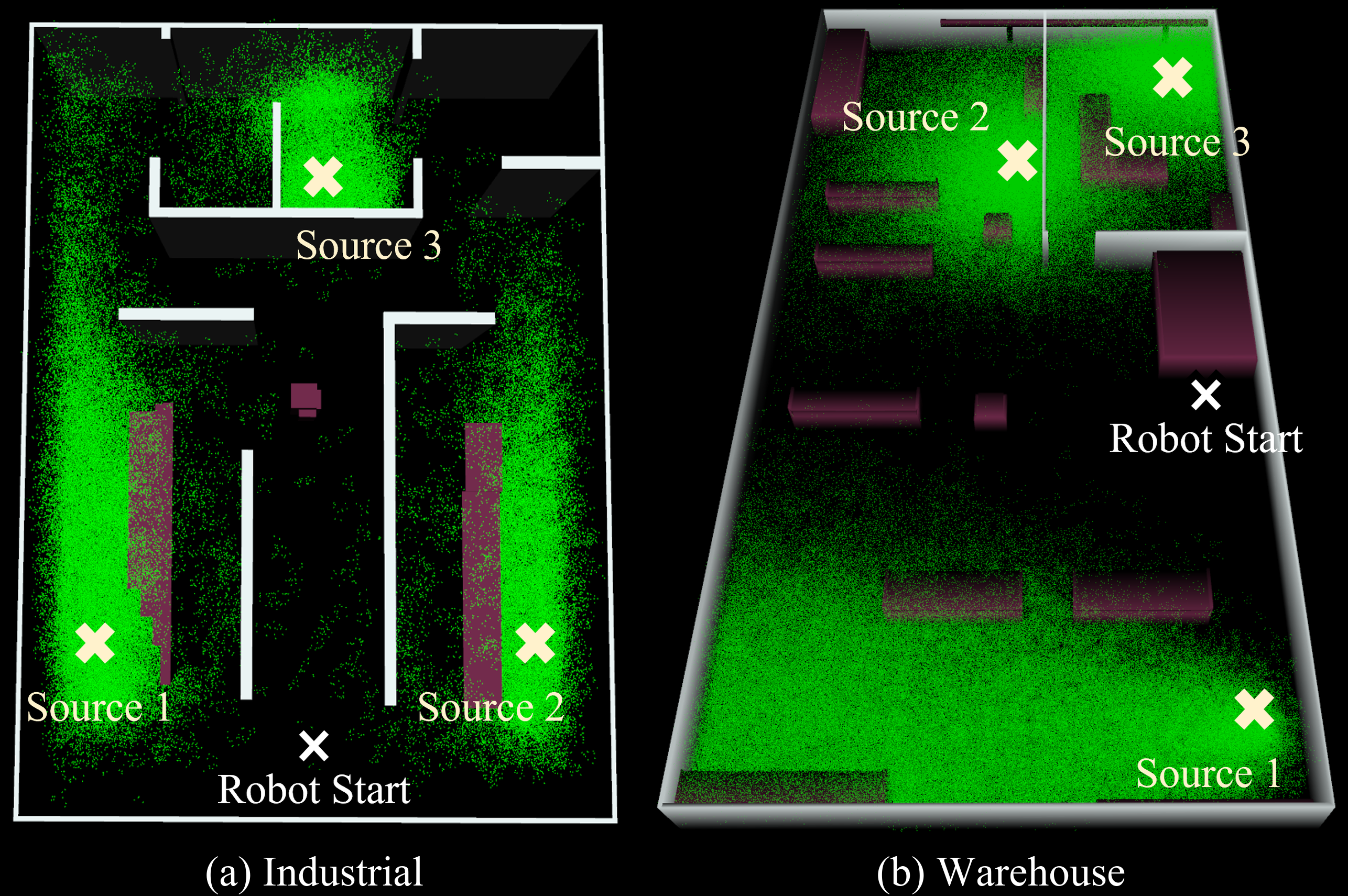}
    \caption{High-fidelity simulation scenarios: (a) the cluttered industrial indoor layout and (b) the warehouse, both with multi-source acetone dispersion. White lines denote walls and magenta blocks storage obstacles; green particles visualise the dispersion field; crosses mark the labelled acetone sources and robot start.}
    \label{fig:hifi_scenario_setup}
\end{figure}

\begin{figure}[!t]
    \centering
    \includegraphics[width=0.92\columnwidth]{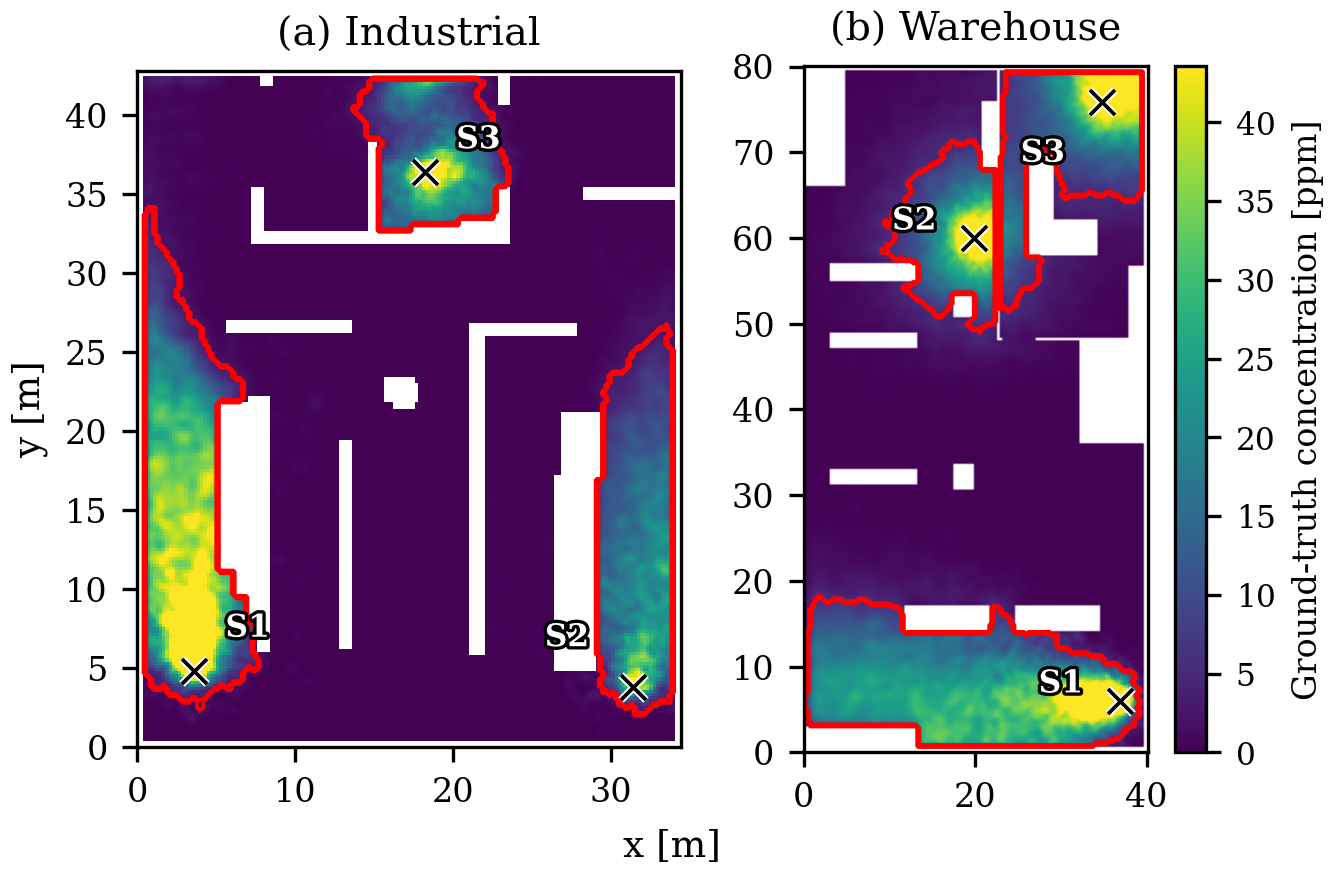}
    \caption{Ground-truth gas distributions of (a) the industrial scenario and (b) the warehouse scenario. Crosses mark the acetone sources (S1--S3, in the order of Table~\ref{tab:sim_parameters}), white regions are obstacles, and the critical gas regions used in the evaluation are outlined in red.}
    \label{fig:gt_two_scenarios}
\end{figure}

\begin{table}[h]
\centering
\footnotesize
\caption{Key Simulation Parameters.}
\label{tab:sim_parameters}
\begin{tabular}{@{}>{\raggedright\arraybackslash}p{0.55\columnwidth}l@{}}
\hline
Parameter & Value (industrial\,/\,warehouse) \\
\hline
\multicolumn{2}{@{}l}{\textbf{Environment \& Gas Sources}} \\[0.1em]
\quad Environment size [m] & $34 \times 43$~/~$80 \times 40$ \\
\quad Release rates S1, S2, S3 [ppm/s] & 1750, 1250, 1000 \\
\quad Filament growth rate [cm$^2$/s] & 15 \\
\quad Wind field [m/s] & 0.3 / 0.1 (steady) \\
\hline
\multicolumn{2}{@{}l}{\textbf{Robot Configuration \& Sensing}} \\[0.1em]
\quad Speed [m/s] & 2.5 / 3.5 \\
\quad Obstacle inflation [m] & 0.2 \\
\quad PID sensor sampling rate [Hz] & 5 \\
\hline
\multicolumn{2}{@{}l}{\textbf{GMRF Hyperparameters}} \\[0.1em]
\quad Lattice resolution [m] & 0.2 \\
\quad Regularisation variance ($\sigma_r^2$) & 3 \\
\quad Observation variance ($\sigma_s^2$) & 10 \\
\quad Time-decay variance ($\sigma_\zeta^2$) & $1 \times 10^{10}$ \\
\quad Default factor variance ($\sigma_d^2$) & 0.001 \\
\hline
\multicolumn{2}{@{}l}{\textbf{Gas Frontiers}} \\[0.1em]
\quad Critical gas threshold Eq.~\eqref{eq:gamma_def} & $\tau_{\text{gas,min}}=2$ ppm, $q_p=10$ \\
\hline
\end{tabular}
\end{table}

The simulated mobile platform travelled at a fixed speed and was equipped with a PID sensor sampling at 5 Hz for gas detection and a standard LiDAR for range measurements. Online occupancy mapping was performed using the open-source SLAM Toolbox \cite{Macenski2021SLAMToolbox}, with obstacles inflated by approximately 0.2 m to ensure robust planning. To isolate planning performance, true localisation was obtained directly from Gazebo and path following was idealised such that the robot tracked planned trajectories perfectly, deliberately eliminating variance from robot dynamics and control errors that would otherwise obscure differences in the planning metrics.

The \gls{gmrf} \cite{Monroy2016} was used as the \gls{gdm} function with the \gls{gbp} solver from~\cite{Rhodes2023STRUC} to achieve real-time updates. Conventional frontier detection was implemented using the \gls{wfd} algorithm~\cite{Keidar2014EfficientFrontier}, while gas frontiers were identified using the proposed \gls{wgfd} algorithm (Section~\ref{sec:gas_frontier}). Table~\ref{tab:sim_parameters} summarises the key simulation parameters.

\subsection{Evaluation Methodology}
\label{sec:eval_methodology}

The \gls{xit} planner was configured with a batch size of $N = 300$ samples per planning iteration in the industrial scenario and $N = 400$ in the warehouse, reflecting its larger extent. At each iteration, up to five goal regions were selected based on the goal-selection policy. For each selected goal region $f \in \mathcal{F}_k^{\mathrm{occ}} \cup \mathcal{F}_k^{\mathrm{gas}}$, the goal set $\mathbf{X}_{\mathrm{goal}}^{(j)}$ was constructed using the $k_n = 3$ nearest samples from $\mathbf{X}_{\mathrm{s}}$ to the frontier centroid. A small uniform distribution component (20\%) was added to the informed \gls{ucb} sampling distribution (Eq.~\eqref{eq:mixture_sampling_density}) to ensure coverage of the entire free space, as described in Section~\ref{subsec:algorithm properties}. All simulations were executed on an AMD Ryzen 9 5950X 16-Core Processor. Table~\ref{tab:xit_parameters} provides a summary of the key \gls{xit} configuration parameters.
\begin{table}[h]
\centering
\footnotesize
\caption{XIT Planner Configuration Parameters.}
\label{tab:xit_parameters}
\begin{tabular}{@{}p{0.55\columnwidth}l@{}}
\hline
Parameter & Value (industrial\,/\,warehouse) \\
\hline
Batch size ($N$) [samples] & 300 / 400 \\
Maximum goal regions & 5 \\
Nearest samples per goal ($k_n$) & 3 \\
Uniform sampling component & 20\% \\
\hline
\end{tabular}
\end{table}

\begin{figure*}[!t]
    \centering
    \includegraphics[width=0.86\textwidth]{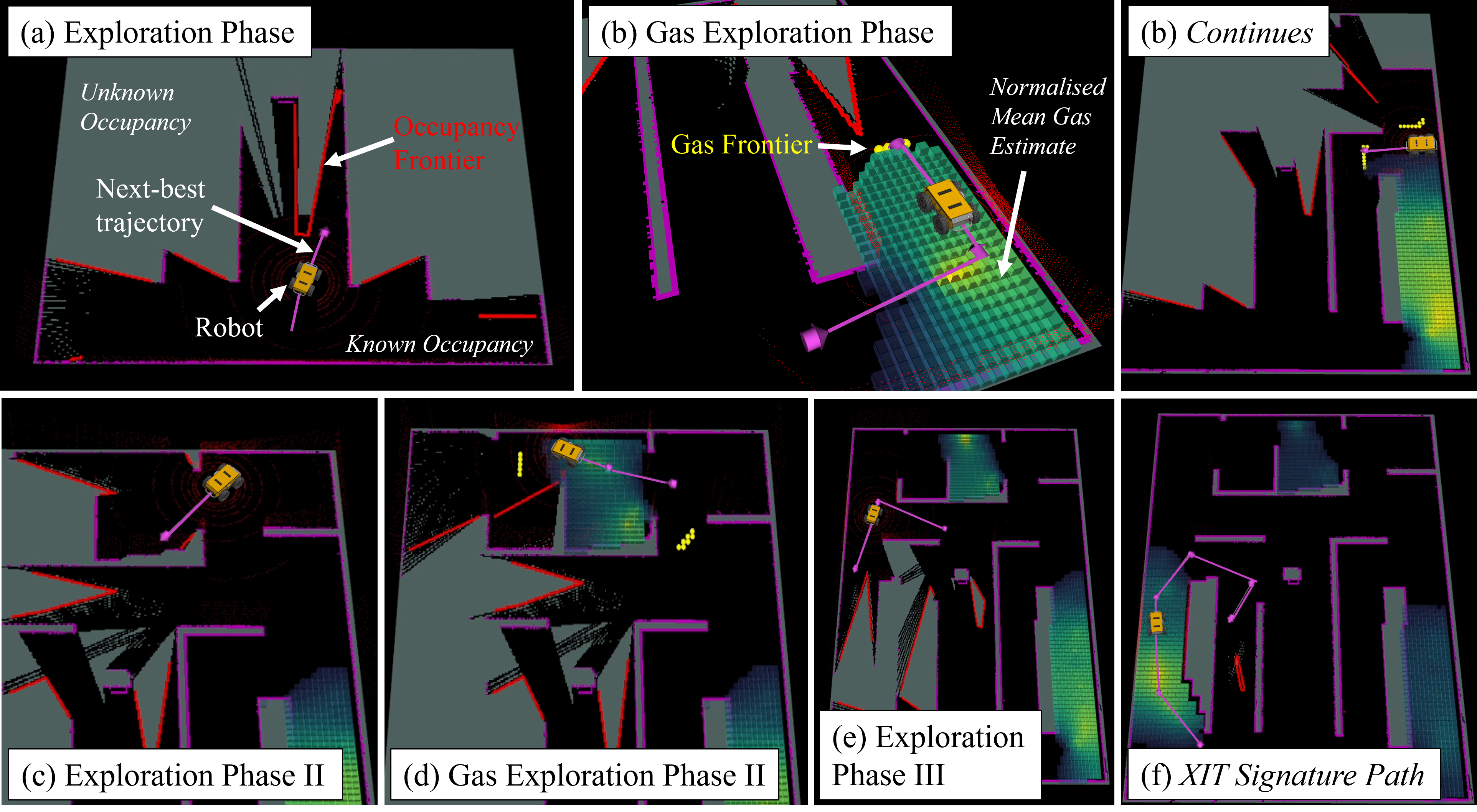}
    \caption{Illustrative run of the active \gls{gdm} process in simulation (industrial scenario) using \gls{xit}-GFF with \gls{ucb} cost. Only the selected \gls{nbt} is shown in each phase for clarity. The normalised posterior mean gas estimate is overlaid using the viridis colormap (yellow indicates higher concentration). (a) Initial structural exploration with occupancy frontiers. (b) Gas frontier detection triggers targeted probing of critical gas boundaries. (c--e) Alternation between occupancy exploration and gas-structure exploration as new critical regions are encountered. (f) Example of a gas-exploitative path that traverses high-concentration regions rather than taking the shortest route.}
    
    \label{fig:activegdm_process}
\end{figure*}

Three \gls{xit} goal-selection policies are tested. \textit{\gls{xit}-F} selects up to 5 occupancy frontiers at random. \textit{\gls{xit}-FGF} queues all available occupancy frontiers first (randomly selecting 5 if more exist), then fills remaining slots with random gas frontiers. \textit{\gls{xit}-GFF} reverses this: gas frontiers first, then occupancy frontiers fill remaining slots.

Two cost formulations are evaluated: \textit{UCB}, the proposed information penalty cost at $\alpha = 0$, which minimised closed-loop mapping RMSE in the ablation of Section~\ref{subsec:alpha_ablation}, for both tree expansion (Eq.~\eqref{eq:edge-cost}) and trajectory selection (Eq.~\eqref{eq:nbt2}), representing gas exploitation planning, and \textit{Euclidean}, distance-only cost for tree expansion. \gls{xit}-GFF and \gls{xit}-FGF are each evaluated under both cost formulations. 
Performance is benchmarked against \emph{RRT* Frontier Exploration}, which uses RRT* with distance-based cost to plan paths to occupancy frontiers, representing standard frontier-based exploration without gas-awareness. Frontier-based exploration and sampling-based planners are the established baselines in adaptive informative planning \cite{POPOVIC2024104727}; the learning-based planners discussed in the introduction and related work \cite{Vashisth2024,Cheng2025TRO,Kulbaka2025IROS} address different settings and are therefore not directly comparable. Note that \gls{xit}-F with Euclidean cost would yield similar performance to this baseline at the scale of this active \gls{gdm} experiment due to the absence of gas-awareness in both approaches, and is therefore not separately evaluated. \gls{xit}-F with \gls{ucb} cost is included, yielding five total \gls{xit} configurations plus the baseline.

Each trial began with the robot at a fixed starting location in an initially unknown environment and ran for a fixed duration of 5 minutes (industrial) and 10 minutes (warehouse), reflecting their different extents, to ensure consistent conditions for comparing the different algorithm configurations. To reduce variance, 20 Monte Carlo simulations were conducted for each configuration in each scenario, with performance averaged across trials.

System performance was evaluated using three metrics. The first two assess \gls{gdm} quality in critical gas regions, defined for evaluation purposes as free-space cells where ground truth concentration exceeds the scenario's evaluation critical gas threshold ($z_{\text{thresh}} = 2.5$ ppm industrial, $5.5$ ppm warehouse). These critical regions, corresponding to areas of significant acetone presence, are visualised spatially in Fig.~\ref{fig:gt_two_scenarios} (outlined in red). This focus on critical regions is particularly important for entropy evaluation, as computing entropy across the entire map would primarily reflect the number of measurements obtained in newly explored areas rather than the quality of gas distribution inference. The metrics are: (i) root mean square error (RMSE), computed as in Eq.~\eqref{eq:rmse}, measuring mapping accuracy against ground truth obtained from a fine lawnmower sweep (Fig.~\ref{fig:gt_two_scenarios}); and (ii) total differential entropy, computed as in Eq.~\eqref{eq:entropy}, quantifying total uncertainty in the gas distribution estimate within these highly concentrated areas. The third metric measures exploration completeness, defined as the fraction of ground-truth map cells whose occupancy state has been identified by the online map at time $k$. Final performance is reported as percentage improvement relative to the RRT* Frontier Exploration baseline.

\begin{figure}[htb]
    \centering
    \includegraphics[width=0.95\columnwidth]{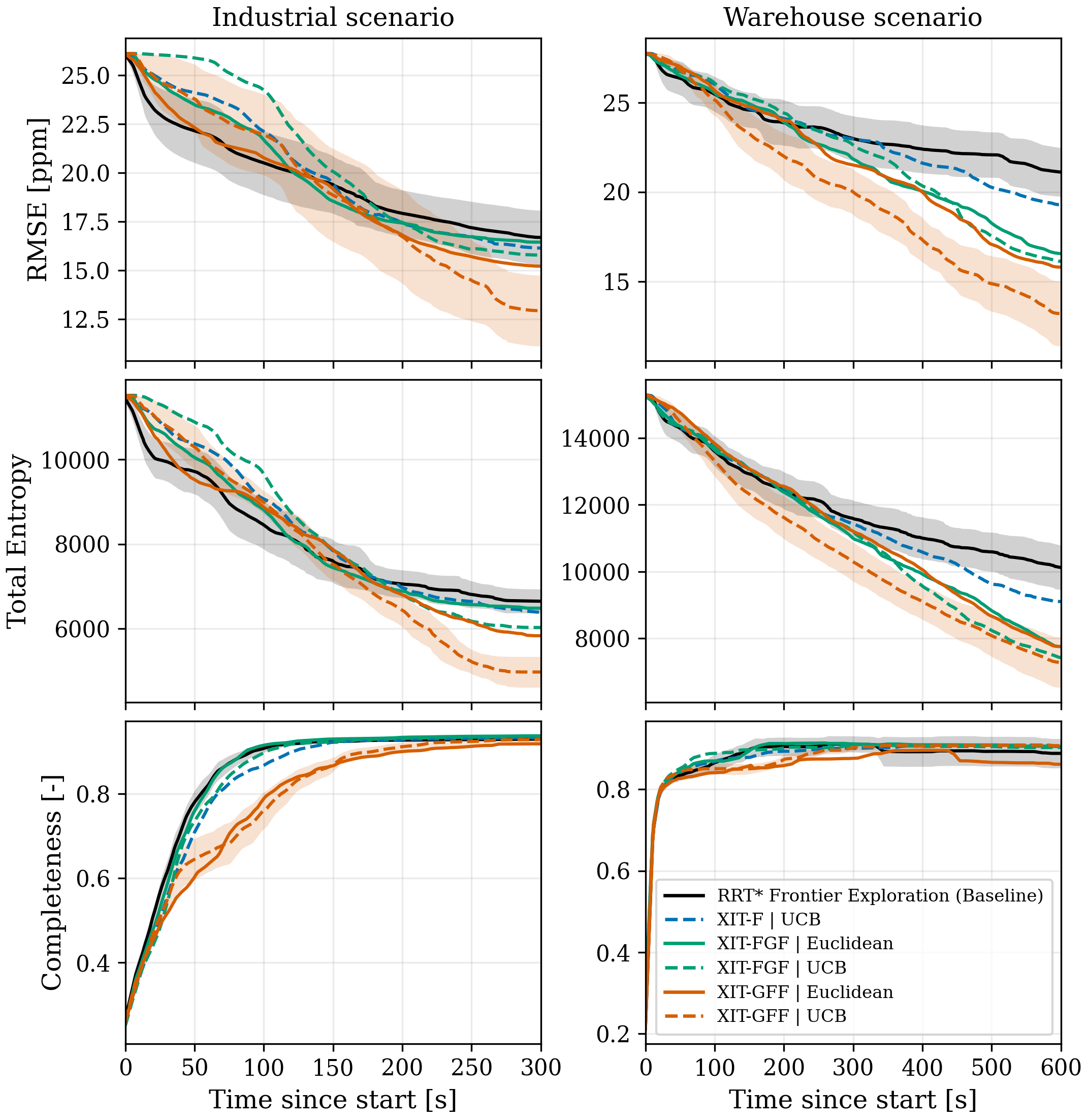}
    \caption{Performance comparison over time in the industrial (left, 300 s missions) and warehouse (right, 600 s missions) scenarios, in terms of RMSE (top), total differential entropy (middle), and map completeness (bottom), for different goal-selection policies and cost formulations. Shaded regions indicate 95\% confidence intervals for \gls{xit}-GFF with \gls{ucb} (best performing) and RRT* Frontier Exploration (baseline).}
    \label{fig:data3plots}
\end{figure}

\subsection{Results}
\label{sec:sim_results}

\begin{table*}[htb]
\centering
\caption{Final GDM Performance at Mission End ($T{=}300$~s Industrial, $T{=}600$~s Warehouse; Mean $\pm$ Std Over 20 Trials).}
\label{tab:hifi_final_metrics}
\footnotesize
\setlength{\tabcolsep}{5.5pt}
\renewcommand{\arraystretch}{1.05}

\begin{tabular}{@{}l l c c c c@{}}
\hline
Method & Cost & RMSE$_T$ [ppm] & RMSE Reduction [\%] & $H_T$ & Entropy Reduction [\%] \\
\hline
\multicolumn{6}{@{}l}{\textit{Industrial scenario}} \\
RRT* Frontier Exploration & Euclidean & $16.7 \pm 2.6$ & $0.0$  & $6652 \pm 524$ & $0.0$ \\
XIT-F & UCB       & $16.1 \pm 2.8$ & $3.3$ & $6395 \pm 862$ & $3.9$ \\
XIT-FGF        & Euclidean & $16.4 \pm 2.7$ & $1.5$  & $6488 \pm 728$ & $2.5$ \\
XIT-FGF        & UCB       & $15.8 \pm 3.1$ & $5.4$ & $6032 \pm 987$ & $9.3$ \\
XIT-GFF        & Euclidean & $15.2 \pm 2.2$ & $8.8$  & $5838 \pm 578$ & $12.2$ \\
XIT-GFF        & UCB       & $12.9 \pm 3.3$ & $22.5$ & $4981 \pm 669$  & $25.1$ \\
\hline
\multicolumn{6}{@{}l}{\textit{Warehouse scenario}} \\
RRT* Frontier Exploration & Euclidean & $21.1 \pm 2.6$ & $0.0$  & $10122 \pm 1279$ & $0.0$ \\
XIT-F & UCB       & $19.3 \pm 2.9$ & $8.7$ & $9099 \pm 1269$ & $10.1$ \\
XIT-FGF        & Euclidean & $16.6 \pm 2.3$ & $21.6$  & $7751 \pm 1012$ & $23.4$ \\
XIT-FGF        & UCB       & $16.1 \pm 4.4$ & $23.7$ & $7425 \pm 1620$ & $26.6$ \\
XIT-GFF        & Euclidean & $15.8 \pm 2.7$ & $25.2$  & $7752 \pm 517$ & $23.4$ \\
XIT-GFF        & UCB       & $13.2 \pm 3.6$ & $37.5$ & $7280 \pm 1489$  & $28.1$ \\
\hline
\multicolumn{6}{@{}l}{\footnotesize Percentage reductions computed relative to RRT* Frontier Exploration baseline.} \\
\end{tabular}

\end{table*}

Figure~\ref{fig:activegdm_process} illustrates a representative run of the \gls{xit}-GFF configuration with the \gls{ucb} cost, showing the alternation between structural exploration and gas-specific probing and exploitation. Initially, the robot explores the environment using occupancy frontiers. Upon encountering critical gas regions, gas frontiers are detected and prioritised due to the policy, leading the robot to probe the boundaries of the critical gas structure it finds itself in. Once local gas structure is resolved, the system returns to occupancy exploration. As the gas map becomes fairly populated, \gls{ucb}-driven gas-exploitative paths toward these goals become more significant, as seen in Fig.~\ref{fig:activegdm_process}f, creating an adaptive exploration and exploitation process.

Figure~\ref{fig:data3plots} presents the temporal evolution of all three metrics averaged over 20 Monte Carlo simulations for each configuration in both scenarios. The \gls{ucb} information penalty cost function achieved better final gas-mapping quality (RMSE and entropy) than the Euclidean cost for the same goal-selection policy, while achieving comparable final exploration coverage. This is consistent with the \(\alpha\) ablation of Section~\ref{subsec:alpha_ablation}: the information-driven cost (\(\alpha=0\)) attains the more accurate gas map, with the trade-off appearing here as slower early coverage rather than increased travel effort.

The RMSE plot (Fig.~\ref{fig:data3plots}, top left) shows that \gls{xit}-GFF with \gls{ucb} cost achieves the greatest reduction in mapping error, reaching below 13 ppm by the end of the mission; this represents a 22.5\% improvement over the RRT* Frontier Exploration baseline (see Table~\ref{tab:hifi_final_metrics} for detailed metrics). In this scenario, structural exploration dominates the early mission, so the baseline initially reduces RMSE fastest; the gas-directed policies overtake in the second half of the mission as their targeted probing of the critical gas regions pays off. The entropy plot (Fig.~\ref{fig:data3plots}, middle left) demonstrates similar trends, with \gls{xit}-GFF using \gls{ucb} cost achieving 25.1\% lower entropy in critical regions. Map completeness (Fig.~\ref{fig:data3plots}, bottom left) shows similar final coverage across all methods, with the frontier-first variants reaching near-complete coverage by the halfway mark and the gas-first policies arriving more gradually, as early gas probing defers some structural exploration, indicating that the gas-frontier mechanism does not compromise final exploration coverage. The primary performance differences lie in gas mapping quality (RMSE and entropy) rather than spatial coverage.

The right column of Fig.~\ref{fig:data3plots} presents the corresponding results for the warehouse scenario. Here the separation from the baseline is more pronounced: the open space is covered within the first two minutes, so the remaining mission rewards targeted probing of the extensive critical gas regions. Every \gls{xit} configuration improves on the baseline, the gas-frontier policies are ahead under both cost formulations, and \gls{xit}-GFF with \gls{ucb} cost again performs best, reducing final RMSE by 37.5\% and entropy by 28.1\% (Table~\ref{tab:hifi_final_metrics}). The larger advantage reflects the warehouse's greater extent and its broad plumes, which reward sustained information-driven probing.

In all \gls{xit} policies, at the start of each run when no gas frontiers exist, occupancy frontiers are automatically prioritised, initiating the environmental exploration phase as illustrated in Fig.~\ref{fig:activegdm_process}. Among the goal-selection policies, queuing gas frontiers first followed by occupancy frontiers (\gls{xit}-GFF) performed best. This result is intuitive: prioritising gas frontiers encourages the robot to thoroughly explore known gas regions before moving to new areas where gas may or may not be present. The gas frontier mechanism provides a structured approach to this, as probed frontiers are removed once local gas structure is resolved, naturally guiding the robot away from already-explored gas regions. This contrasts with concentration-based or uncertainty-based methods, which can exhibit suboptimal looping by revisiting high-concentration or high-variance cells without clear termination criteria, as observed in \cite{Ercolani2022}. While \cite{Ercolani2022} addressed this by discretising the space with K-means clustering and adaptively sampling within predefined clusters in obstacle-free environments, our approach must simultaneously explore an unknown environment while mapping gas. We integrate occupancy frontiers for structural exploration with gas frontiers for targeted probing, where the \gls{ucb} cost function drives gas-exploitative paths and the system naturally transitions between objectives based on emerging critical gas regions rather than predefined spatial partitions. Between the other \gls{xit} policies, \gls{xit}-FGF prioritises occupancy frontiers first, emphasising a more exploration-focused approach that may delay thorough investigation of known gas structures. Notably, in both policies, when both gas frontiers and occupancy frontiers are in the queue, the most gas-exploitative trajectory is identified via the \gls{ucb} trajectory selection function (Eq.~\eqref{eq:nbt2}). It was observed that this often favours trajectories to gas frontiers due to the higher \gls{ucb} density between the robot and these boundaries, which likely explains why \gls{xit}-FGF with \gls{ucb} cost achieves a stronger entropy reduction in critical regions than \gls{xit}-F with \gls{ucb} cost, which does not consider gas frontiers as a goal source.

Table~\ref{tab:hifi_final_metrics} summarises the final performance metrics at the end of each mission, confirming that \gls{xit}-GFF with \gls{ucb} cost achieves the best final performance on both gas-mapping metrics in both scenarios.

\section{Experimental Trial}

\begin{figure}[t]
    \centering
    \includegraphics[width=0.9\columnwidth]{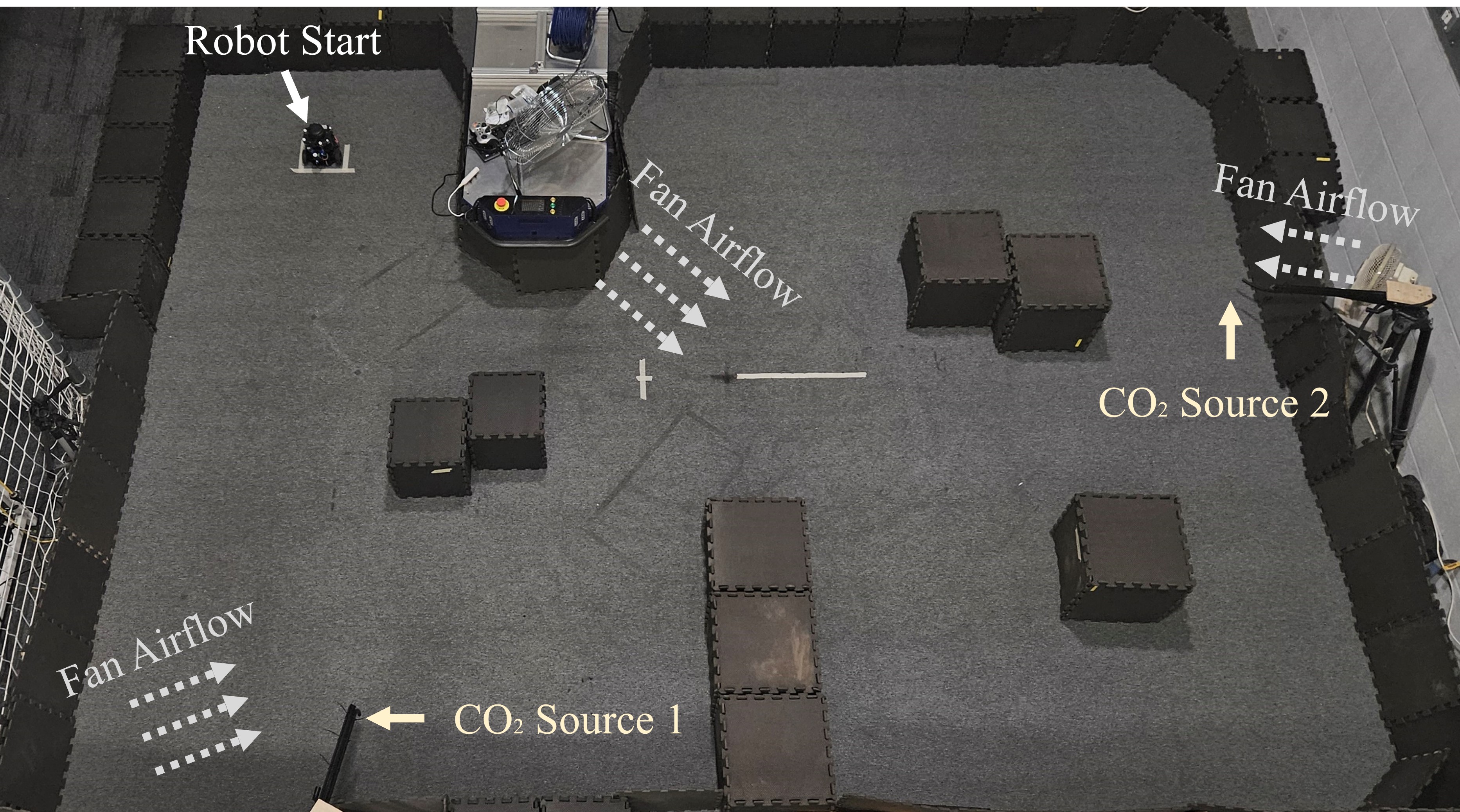}
    \par\vspace{0.5ex}
    (a)
    \par\vspace{1.5ex}

    \begin{minipage}[t]{0.48\columnwidth}
        \centering
        \includegraphics[width=\linewidth]{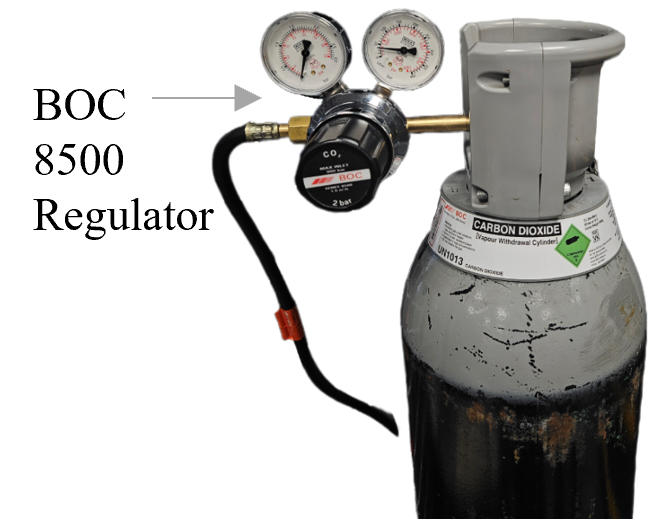}
        \par\vspace{0.5ex}
        (b)
    \end{minipage}\hfill
    \begin{minipage}[t]{0.48\columnwidth}
        \centering
        \includegraphics[width=\linewidth]{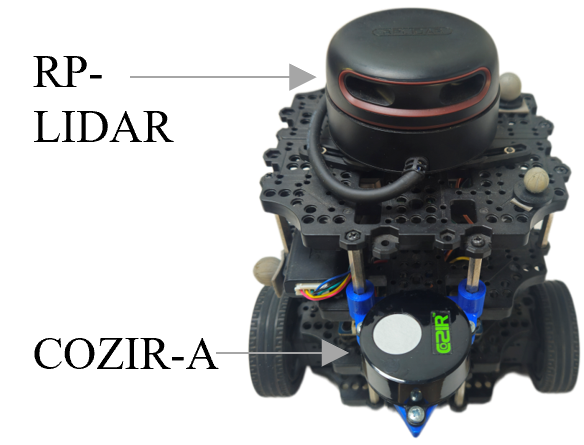}
        \par\vspace{0.5ex}
        (c)
    \end{minipage}

    \caption{
        Experimental setup. 
        (a) Top-down view of the 5.0 m × 3.7 m arena showing airflow directions from fans (dashed arrows) and CO$_2$ release nozzles (Source 1 and Source 2). Note that the two fans on the left are more powerful than the single and smaller fan on the right, creating asymmetric dispersion patterns. 
        (b) BOC 8500 regulator attached to the pressurised CO$_2$ cylinder used for controlled gas release. 
        (c) TurtleBot3 platform equipped with a COZIR-A CO$_2$ sensor and RP-LiDAR for gas mapping and navigation.
    }
    \label{fig:exp_setup}
\end{figure}

\begin{figure}
    \centering
    \includegraphics[width=0.88\columnwidth]{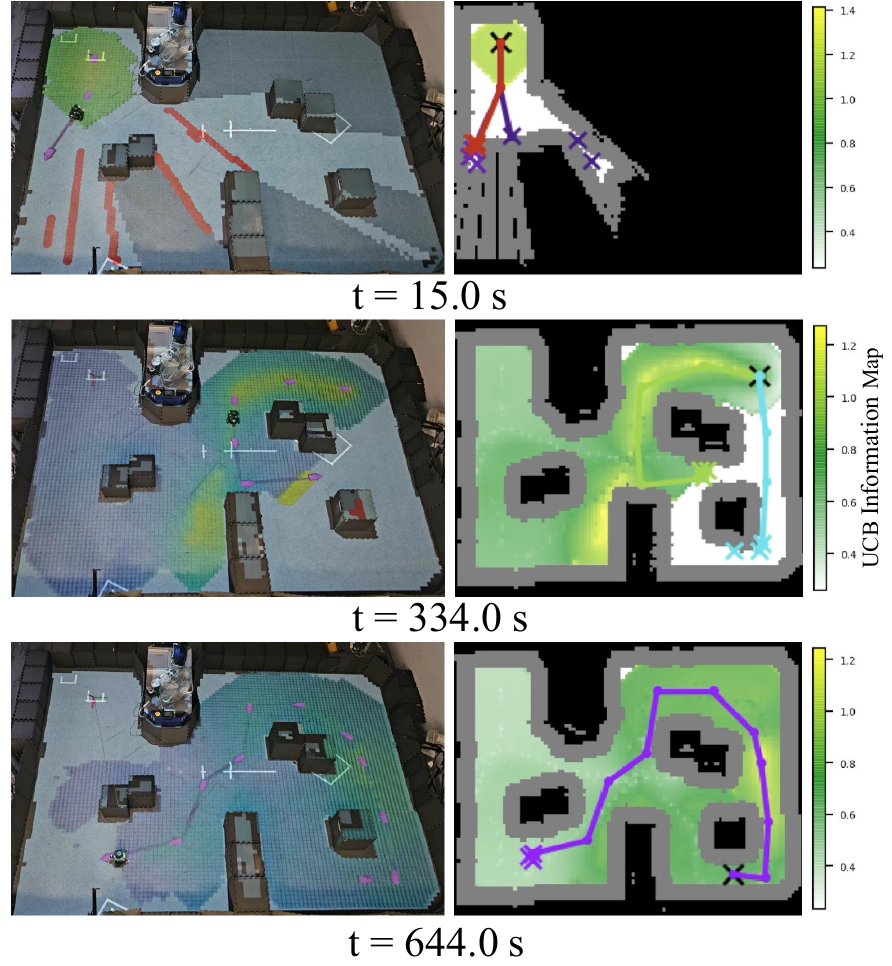}
\caption{
Progression of the real-world experiment using the \gls{xit} planner in the scaled indoor gas-leak environment. 
Each row shows the top-down view of the arena (left) with the normalised gas mean map projected on the floor, where red, yellow, and purple lines indicate frontiers, gas frontiers, and the \gls{nbt}, respectively. 
The corresponding \gls{xit} paths and \gls{ucb} information field are shown on the right with obstacle inflation in grey. 
From top to bottom: \gls{xit} begins in a pure exploration phase, as the starting location lies outside the critical gas region, proceeds through gas-exploration phases around both sources (the second shown in the middle), and concludes with a gas-exploitative trajectory that passes downwind through both gas regions. A video of the experiment can be found in the supplementary material.
}
    \label{fig:exp_progression}
\end{figure}

A physical experiment was conducted to validate the practicality of the proposed \gls{xit} planner and broader active \gls{gdm} framework (Fig.~\ref{fig:overview}) in a scaled, planar indoor environment (Fig.~\ref{fig:exp_setup}) designed to emulate a confined gas-leak scenario. The arena measured 5.0 m × 3.7 m and was enclosed with barriers to mimic the restricted airflow of an indoor space. Two CO$_2$ sources were installed at opposite sides of the arena, each connected to a pressurised CO$_2$ cylinder via a BOC Series 8500 multi-stage regulator (Fig.~\ref{fig:exp_setup}b). CO$_2$ was selected as the tracer gas due to its density being greater than that of air, which causes it to settle near ground level, creating realistic stratification in the scaled environment. Three fans generated controlled airflow to produce a complex, multi-region gas plume pattern that challenged the planner’s ability to identify and exploit informative regions.

The mobile platform was a TurtleBot3 (Fig.~\ref{fig:exp_setup}c) equipped with a COZIR-A infrared CO$_2$ sensor mounted at 0.15 m height to sample near the gas layer. The robot travelled at 0.15 m/s with obstacle inflation of 0.2 m for safe navigation. A Vicon motion-capture system provided high-accuracy odometry, while the occupancy grid map was generated online using SLAM Toolbox \cite{Macenski2021SLAMToolbox} and gas distribution mapping was performed by \gls{gmrf} \cite{Monroy2016} with the \gls{gbp} solver \cite{Rhodes2023STRUC}, as in the simulation study. Path following employed a simple proportional controller that maintains constant forward velocity while adjusting angular velocity based on heading error to the next waypoint. Due to the slow response of the COZIR-A sensor, the robot executed a stop-and-sample routine at each vertex of the \gls{nbt}, remaining stationary for 5 seconds and using the averaged reading for \gls{gdm}. This contrasts with the faster acetone PID sensor used in the simulation study, where continuous sensing along the trajectory was possible.

The experiment employed \gls{xit}-GFF with the \gls{ucb} cost function ($\alpha = 0$) for both tree expansion (Eq.~\eqref{eq:edge-cost}) and trajectory selection (Eq.~\eqref{eq:nbt2}), the most successful configuration from simulation, with up to five candidate goal regions evaluated per planning iteration. A fixed threshold $\tau_{\text{gas}} = 600$ ppm was used in place of the dynamic threshold (Eq.~\eqref{eq:gamma_def}) due to the limited spatial extent of the scaled environment and the slow response characteristics of available CO$_2$ sensors when detecting concentration changes in the range just above ambient background levels (approximately 400 ppm). The experiment started 30 s after initiating gas release to allow plume formation and continued for approximately eleven minutes. Once all frontiers and gas frontiers were resolved, \gls{xit} entered a fallback mode in which goals were sampled randomly in free space. The run terminated after two consecutive \gls{xit} trajectories failed to reveal any new frontiers or gas frontiers.

During the run, the robot began in an unknown region outside the gas plume and initially navigated toward frontiers to build the occupancy map, as shown at the top of Fig.~\ref{fig:exp_progression}. As the robot approached the first source, gas frontiers were established, prompting \gls{xit} to probe the local gas boundaries around the source. The robot then transitioned to an exploration phase that led it near the second source, where a subsequent gas-exploration phase occurred within a more complex dispersion region (Fig.~\ref{fig:exp_progression}, middle row). Throughout the experiment, \gls{xit} generated trajectories through high-\gls{ucb} areas, often gas-exploitative regions, which led it to revisit high-concentration zones where the route allowed. An example of this behaviour is shown in the bottom row of Fig.~\ref{fig:exp_progression}, where the robot passes through both regions downwind of the sources on its way to the final goal. The final gas map (Fig.~\ref{fig:final_gas_map}) showed CO$_2$ levels exceeding 2000~ppm in some regions, forming extended gas zones with higher concentrations near the release points. Notably, higher concentrations are observed near Source 2 on the right side of the arena, consistent with the asymmetric airflow pattern created by the two more powerful fans on the left side directing flow toward this region.

\begin{figure}[t]
    \centering
    \includegraphics[width=0.9\columnwidth]{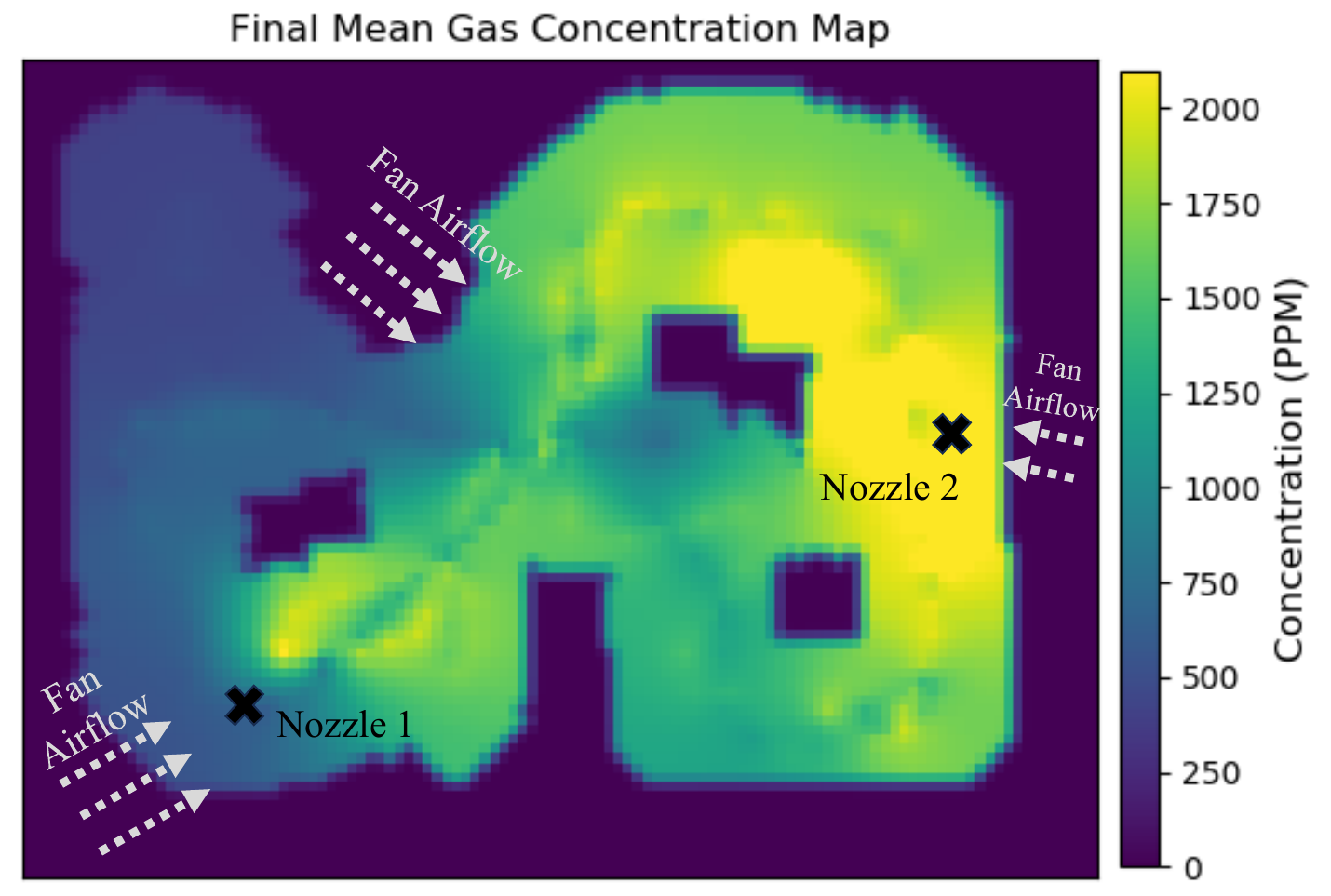}
    \caption{
    Final mean gas concentration map obtained in the real-world experiment using the \gls{xit} planner within the active \gls{gdm} framework. Source locations and airflow directions (dashed arrows) are indicated as in Fig.~\ref{fig:exp_setup}a.
}
    \label{fig:final_gas_map}
\end{figure}

\section{Conclusion}

This paper presents a unified active \gls{gdm} framework for discovering unknown, cluttered environments and mapping the gas distribution therein. Our sampling-based planner, \gls{xit}, grows a single shared tree toward environmental and gas frontiers, coupling progress toward unresolved regions with gas exploitation en route under an information-field-guided cost.  We provide probabilistic completeness and asymptotic optimality proofs under an idealised sampling model and characterise \gls{xit} empirically in synthetic simulations.

In high-fidelity Monte Carlo simulations of two large indoor scenarios, the \gls{ucb}-based cost function consistently outperforms distance-based planning in gas mapping quality (RMSE and entropy) while maintaining similar final coverage. The gas frontier, introduced in this paper, provides an intuitive mechanism for gas structure exploration. Prioritising gas frontiers over occupancy frontiers (\gls{xit}-GFF) is the most effective goal-selection policy tested. As probed frontiers are removed once the local gas structure is resolved, the robot is drawn toward unexplored gas regions or unknown space rather than revisiting known ones, helping to mitigate the looping behaviour of methods that select one sampling location at a time. A scaled indoor CO$_2$ release experiment demonstrates the practicality of the framework on a physical robot in a two-source scenario with asymmetric airflow and overlapping plumes.

Although developed for \gls{gdm}, the framework's structure, combining information-guided sampling, multi-objective tree expansion, and frontier-based goal selection, provides an adaptable template for robotic information gathering tasks in unknown environments that face the exploration-exploitation trade-off.

{\footnotesize
\bibliographystyle{IEEEtranN}
\bibliography{references}
}

\end{document}